\documentclass{article}

\PassOptionsToPackage{numbers,sort}{natbib}

     \usepackage[final]{neurips_2019}

\usepackage[utf8]{inputenc} 
\usepackage[T1]{fontenc}    
\usepackage{hyperref}       
\usepackage{url}            
\usepackage{booktabs}       
\usepackage{amsfonts}       
\usepackage{nicefrac}       
\usepackage{microtype}      
\usepackage{graphicx}
\usepackage{subfigure}
\usepackage{xcolor}
\usepackage{algorithm}
\usepackage{algorithmic}
\usepackage{amsmath}
\usepackage{amsthm}
\usepackage{MnSymbol}
\usepackage{float}
\usepackage{dsfont}
\usepackage{array}

\newcolumntype{P}[1]{>{\centering\arraybackslash}p{#1}}

\newcommand{\E}[1]{\mathbb{E}\left[ #1 \right]}

\renewcommand{\P}[1]{\mathbb{P}\left[ #1 \right]}
\newcommand{\A}{\mathcal{A}}
\newcommand{\eps}{\varepsilon}
\newcommand{\lmax}{L_{\max}}
\newcommand{\lmin}{L_{\min}}
\renewcommand{\L}{\mathcal{L}}
\newcommand{\algo}{\textsc{UVGausstimate}}
\newcommand{\algok}{\textsc{KVGausstimate}}
\newcommand{\algoone}{\textsc{1RoundUVGausstimate}}
\newcommand{\algokone}{\textsc{1RoundKVGausstimate}}
\newcommand{\rrone}{\textsc{RR1}}
\newcommand{\rronekone}{\textsc{1RoundKVRR2}}
\newcommand{\aggone}{\textsc{Agg1}}
\newcommand{\aggonek}{\textsc{KVAgg1}}
\newcommand{\hatmuone}{\textsc{EstMean}}
\newcommand{\hatsigma}{\textsc{EstVar}}
\newcommand{\rrtwo}{\textsc{UVRR2}}
\newcommand{\rrtwok}{\textsc{KVRR2}}
\newcommand{\aggtwok}{\textsc{KVAgg2}}
\newcommand{\rrthree}{\textsc{1RoundUVRR2}}
\newcommand{\ind}{\mathds{1}}

\newcommand{\est}[1]{\ensuremath{\mathsf{Estimate}}\left( #1 \right)}
\newcommand{\estno}{\mathsf{Estimate}}
\newcommand{\test}[1]{\ensuremath{\mathsf{Test}}\left( #1 \right)}
\newcommand{\testno}{\mathsf{Test}}

\newcommand{\tv}[2]{\Vert #1 - #2 \Vert_{TV}}

\renewcommand{\algorithmiccomment}[1]{\bgroup\hfill$\rhd$~#1\egroup}

\renewcommand{\mod}{\text{ mod }}

\newcommand{\Lap}[1]{\ensuremath{\mathsf{Lap}}\left( #1 \right)}
\newcommand{\sgn}{\text{sgn}}
\newcommand{\poly}{\text{poly}}
\newcommand{\erf}{\text{erf}}

\newtheorem{theorem}{Theorem}[section]
\newtheorem{lemma}[theorem]{Lemma}
\newtheorem{definition}[theorem]{Definition}

\newtheorem{fact}[theorem]{Fact}

\title{Locally Private Gaussian Estimation}

\author{Matthew Joseph \thanks{A portion of this work was done while at Microsoft Research Redmond.} \\
University of Pennsylvania \\
\texttt{majos@cis.upenn.edu} \And
Janardhan Kulkarni \\
Microsoft Research Redmond \\
\texttt{jakul@microsoft.com} \AND
Jieming Mao \thanks{This work done while at the Warren Center, University of Pennsylvania.} \\
Google Research New York \\
\texttt{maojm@google.com} \And 
Zhiwei Steven Wu \thanks{A portion of this work was done while at Microsoft Research New York.} \\
University of Minnesota \\
\texttt{zsw@umn.edu}}

\begin{document}

\maketitle

\begin{abstract}
	We study a basic private estimation problem: each of $n$ users draws a single i.i.d. sample from an unknown Gaussian distribution $N(\mu,\sigma^2)$, and the goal is to estimate $\mu$ while guaranteeing \emph{local differential privacy} for each user. As minimizing the number of rounds of interaction is important in the local setting, we provide \emph{adaptive} two-round solutions and \emph{nonadaptive} one-round solutions to this problem. We match these upper bounds with an information-theoretic lower bound showing that our accuracy guarantees are tight up to logarithmic factors for all sequentially interactive locally private protocols.
\end{abstract}

\section{Introduction}
\label{sec:intro}
Differential privacy is a formal algorithmic guarantee that no single input has a large effect on the output of a computation. Since its introduction~\cite{DMNS06}, a rich line of work has made differential privacy a compelling privacy guarantee (see~\citet{DR14} and~\citet{V17} for surveys), and deployments of differential privacy now exist at many organizations, including Apple~\cite{A17}, Google~\cite{EPK14,BEMMR+17}, Microsoft~\cite{DKY17}, Mozilla~\cite{AKZHL17}, and the US Census Bureau~\cite{A16, KCKHM18}.

Much recent attention, including almost all industrial deployments, has focused on a variant called \emph{local differential privacy}~\cite{DMNS06, BNO08, KLNRS11}. In the local model private data is distributed across many users, and each user privatizes their data \emph{before} the data is collected by an analyst. Thus, as any locally differentially private computation runs on already-privatized data, data contributors need not worry about compromised data analysts or insecure  communication channels. In contrast, (global) differential privacy assumes that the data analyst has secure, trusted access to the unprivatized data.

However, the stronger privacy guarantees of the local model come at a price. For many problems, a locally private solution requires far more samples than a globally private solution~\cite{KLNRS11, DJW13, U18, DF18}. 
Here, we study the basic problem of locally private Gaussian estimation: given $n$ users each holding an i.i.d. draw from an unknown Gaussian distribution $N(\mu,\sigma^2)$, can an analyst accurately estimate the mean $\mu$ while guaranteeing local differential privacy for each user? 

On the technical front, locally private Gaussian estimation captures two general challenges in locally private learning. First, since data is drawn from a Gaussian, there is no a priori (worst-case) bound on the scale of the observations. Naive applications of standard privatization methods like Laplace and Gaussian mechanisms must add noise proportional to the worst-case scale of the data and are thus infeasible. Second, protocols requiring many rounds of user-analyst interaction are difficult to implement in real-world systems and may incur much longer running times. Network latency as well as server and user liveness constraints compound this difficulty~\cite{STU17}. It is therefore desirable to limit the number of \emph{rounds of interaction} between users and the data analyst. Finally, besides being a fundamental learning problem, Gaussian estimation has several real-world applications (e.g. telemetry data analysis~\cite{DKY17}) where one may assume that users' behavior follows a Guassian distribution.

\subsection{Our Contributions}
We divide our solution to locally private Gaussian estimation into two cases. In the first case, $\sigma$ is known to the analyst, and in the second case $\sigma$ is unknown but bounded in known $[\sigma_{\min}, \sigma_{\max}]$. For each case, we provide an $(\eps,0)$-locally private adaptive two-round protocol and nonadaptive one-round protocol\footnote{As ``adaptive'' and ``nonadaptive'' are implicit in ``two-round'' and ``one-round'', we often omit these terms.}. Our privacy guarantees are worst-case; however, when $x_1, \ldots, x_n \sim N(\mu,\sigma^2)$ we also get the following accuracy guarantees. 

\begin{theorem}[Informal]
	When $\sigma$ is known, and $n$ is sufficiently large, there exists two-round protocol  outputting $\hat \mu$ such that $|\hat \mu - \mu| = O\left(\tfrac{\sigma}{\eps}\sqrt{\tfrac{\log(1/\beta)}{n}}\right)$ with probability $1-\beta$, and there exists one-round protocol outputting $\hat \mu$ such that $|\hat \mu - \mu| =  O\left(\tfrac{\sigma}{\eps}\sqrt{\tfrac{\log(1/\beta)\sqrt{\log(n)}}{n}}\right)$ with probability $1-\beta$.
\end{theorem}

\begin{theorem}[Informal]
	When $\sigma$ is unknown but bounded in known $[\sigma_{\min},\sigma_{\max}]$, and $n$ is sufficiently large, there exists two-round protocol outputting $\hat \mu$ such that $|\hat \mu - \mu| = O\left(\tfrac{\sigma}{\eps}\sqrt{\tfrac{\log(1/\beta)\log(n)}{n}}\right)$ with probability $1-\beta$, and there exists one-round protocol outputting $\hat \mu$ such that  $|\hat \mu - \mu| = O\left(\tfrac{\sigma}{\eps}\sqrt{\tfrac{\log([\sigma_{\max}/\sigma_{\min}] + 1)\log(1/\beta)\log^{3/2}(n)}{n}}\right)$ with probability $1-\beta$.
\end{theorem}

All of our protocols are \emph{sequentially interactive}~\cite{DJW13}: each user interacts with the protocol at most once. We match these upper bounds with a lower bound showing that our results are tight for all sequentially interactive locally private protocols up to logarithmic factors. We obtain this result by introducing tools from the \emph{strong data processing inequality} literature~\cite{BGMNW16, R16}. Using subsequent work by~\citet{JMNR19}, we can also extend this lower bound to fully interactive protocols.

\begin{theorem}[Informal]
\label{thm:informal_lower}
	For a given $\sigma$, there does not exist an $(\eps,\delta)$-locally private protocol $\A$ such that for any	$\mu = O\left(\tfrac{\sigma}{\eps}\sqrt{\tfrac{1}{n}}\right)$, given $x_1, \ldots, x_n \sim N(\mu,\sigma^2)$,  $\A$ outputs estimate $\hat \mu$ satisfying $|\hat \mu - \mu| = o\left(\tfrac{\sigma}{\eps}\sqrt{\tfrac{1}{n}}\right)$ with probability $\geq 15/16$. 
\end{theorem}

\subsection{Related Work}
\label{sec:rel_work}
Several works have already studied differentially private versions of various statistical tasks, especially in the global setting.~\citet{KV17} and~\citet{KLSU18} consider similar versions of Gaussian estimation under global differential privacy, respectively in the one-dimensional and high-dimensional cases. For both the known and unknown variance cases,~\citet{KV17} offer an $O\left(\sigma\sqrt{\tfrac{\log(1/\beta)}{n}} + \tfrac{\poly \log(1/\beta)}{\eps n}\right)$ accuracy upper bound for estimating $\mu$. Since an $\Omega\left(\sigma\sqrt{\tfrac{\log(1/\beta)}{n}}\right)$ accuracy lower bound holds even without privacy, our upper and lower bounds show that local privacy adds a roughly $\sqrt{n}$ accuracy cost over global privacy.

In concurrent independent work,~\citet{GRS18} also study locally private Gaussian estimation. We match or better their accuracy results with much lower round complexity. They provide adaptive protocols for the known- and unknown-$\sigma$ settings, with the latter protocol having round complexity $T$ as large as $\Omega(n)$, linear in the number of users. In contrast, we provide both adaptive and nonadaptive solutions, and our protocols all have round complexity $T \leq 2$. A full comparison appears in Figure~\ref{fig:table}.

\begin{figure*}[h]
	\begin{tabular}{|P{1.75cm}|P{5.1cm}|P{5.9cm}|}
	\hline
	& \citet{GRS18}  &  This Work  \\
	\hline
	 Setting  &  Accuracy $\alpha$, Round Complexity $T$  &  Accuracy $\alpha$, Round Complexity $T$ \\
	\hline
	 \vspace{3.5pt} Known $\sigma$, adaptive  & \vspace{0.5pt} $\alpha = O\left(\frac{\sigma}{\eps}\sqrt{\frac{\log\left(\tfrac{1}{\beta}\right)\log\left(\tfrac{n}{\beta}\right)\log\left(\tfrac{1}{\delta}\right)}{n}}\right)$ \newline $T = 2$  &  \vspace{0.5pt} $\alpha = O\left(\frac{\sigma}{\eps}\sqrt{\frac{\log\left(\tfrac{1}{\beta}\right)}{n}}\right)$ \newline $T = 2$  \\
	\hline
	 \vspace{4.5pt} Known $\sigma$, nonadaptive  &  \vspace{13pt} -- &  \vspace{0.5pt} $\alpha = O\left(\frac{\sigma}{\eps}\sqrt{\frac{\log\left(\tfrac{1}{\beta}\right)\sqrt{\log(n)}}{n}}\right)$ \newline $T=1$  \\
	\hline
	 \vspace{2.5pt} Unknown $\sigma$, adaptive  &  \vspace{0.5pt} $\alpha = O\left(\frac{\sigma}{\eps}\sqrt{\frac{\log\left(\tfrac{1}{\beta}\right)\log\left(\tfrac{n}{\beta}\right)\log\left(\tfrac{1}{\delta}\right)}{n}}\right)$ \newline $T =  \Omega\left(\log\left(\frac{R}{\sigma_{\min}}\right)\right)$  &  \vspace{0.5pt} $\alpha = O\left(\frac{\sigma}{\eps}\sqrt{\frac{\log\left(\tfrac{1}{\beta}\right)\log\left(n\right)}{n}}\right)$ \newline $T = 2$  \\
	\hline
	 \vspace{3.5pt} Unknown $\sigma$, nonadaptive  &  \vspace{13pt}--  &  \vspace{0.5pt} $\alpha = O\left(\frac{\sigma}{\eps}\sqrt{\frac{\log\left(\tfrac{\sigma_{\max}}{\sigma_{\min}}+1\right)\log\left(\tfrac{1}{\beta}\right)\log^{3/2}\left(n\right)}{n}}\right)$ \newline $T = 1$  \\
	\hline
	\end{tabular} 
	\caption{A comparison of upper bounds in~\citet{GRS18} and here. In all cases, ~\citet{GRS18} use $(\eps,\delta)$-locally private algorithms and we use $(\eps,0)$. Here, $R$ denotes an upper bound on both $\mu$ and $\sigma$. In our setting, the upper bound on $\mu$ is $O(2^{n\eps^2/\log(n/\beta)})$, leading the unknown variance protocol of~\citet{GRS18} to round complexity potentially as large as $\tilde \Omega(n\eps^2/\log(1/\beta))$.
\label{fig:table}}
\end{figure*}

\citet{GRS18} also prove a tight lower bound for nonadaptive protocols that can be extended to sequentially interactive protocols. We provide a lower bound that is tight for sequentially interactive protocols up to logarithmic factors, and we depart from previous local privacy lower bounds by introducing tools from the \emph{strong data processing inequality} (SDPI) literature~\cite{BGMNW16, R16}. This approach uses an SDPI to control how much information a sample gives about its generating distribution, then uses existing local privacy results to bound the mutual information between a sample and the privatized output from that sample. Subsequent work by~\citet{DR19} generalizes the SDPI framework to prove lower bounds for a broader class of problems in local privacy. They also extend the SDPI framework to prove lower bounds for fully interactive algorithms.

\section{Preliminaries}
\label{sec:prelims}
We consider a setting where, for each $i \in [n]$, user $i$'s datum is a single draw from an unknown Gaussian distribution, $x_i \sim N(\mu,\sigma^2)$, and these draws are i.i.d. In our communication protocol, users may exchange messages over public channels with a single (possibly untrusted) central analyst.\footnote{The notion of a central coordinating analyst is only a useful simplification. As the analyst has no special powers or privileges, any user, or the protocol itself, can be viewed as playing the same role.} The analyst's task is to accurately estimate $\mu$ while guaranteeing local differential privacy for each user.

To minimize interaction with any single user, we restrict our attention to \emph{sequentially interactive} protocols. In these protocols, every user sends at most a single message in the entire protocol. We also study the \emph{round complexity} of these interactive protocols. Formally, one round of interaction in a protocol consists of the following two steps: 1) the analyst selects a subset of users $S\subseteq [n]$, along with a set of randomizers $\{Q_i \mid i \in S\}$, and 2) each user $i$ in $S$ publishes a message $y_i = Q_i(x_i)$.

A randomized algorithm is \emph{differentially private} if arbitrarily changing a single input does not change the output distribution ``too much''. This preserves privacy because the output distribution is insensitive to any change of a single user's data. We study a stronger privacy guarantee called \emph{local differential privacy}. In the local model, each user $i$ computes their message using a \emph{local randomizer}. A local randomizer is a differentially private algorithm taking a single-element database as input.

\begin{definition}[Local Randomizer]
A randomized function $Q_i \colon X\rightarrow Y$ is an $(\eps,\delta)$-local randomizer if, for every pair of observations $x_i, x_i'\in X$ and any $S\subseteq Y$, $\Pr[Q_i(x_i)\in S] \leq e^{\eps} \Pr[Q_i(x_i') \in S] + \delta.$
\end{definition}  

A protocol is locally private if every user computes their message using a local randomizer. In a sequentially interactive protocol, the local randomizer for user $i$ may be chosen adaptively based on previous messages $z_1, \ldots, z_{i-1}$. However, the choice of randomizer cannot be based on user $i$'s data.

\begin{definition}
	A sequentially interactive protocol $\A$ is \emph{$(\eps,\delta)$-locally private} for private user data $\{x_1,\ldots,x_n\}$ if, for every user $i \in [n]$, the message $Y_i$ is computed using an $(\eps,\delta)$-local randomizer $Q_i$. When $\delta > 0$, we say $\A$ is \emph{approximately}     locally private. If $\delta = 0$, $\A$ is \emph{purely} locally private.
\end{definition}

\section{Estimating $\mu$ with Known $\sigma$}
\label{sec:upper_known}
We begin with the case where $\sigma^2$ is known (shorthanded ``KV''). In Section~\ref{subsec:kv_two}, we provide a protocol \algok~that requires two rounds of analyst-user interaction. In Section~\ref{subsec:kv_one}, we provide a protocol \algokone~achieving a weaker accuracy guarantee in a single round. All omitted pseudocode and proofs appear in the Supplement.

\subsection{Two-round Protocol~\algok}
\label{subsec:kv_two}
In \algok~the users are split into halves $U_1$ and $U_2$. In round one, the analyst queries users in $U_1$ to obtain an $O(\sigma)$-accurate estimate $\hat \mu_1$ of $\mu$. In round two, the analyst passes $\hat \mu_1$ to users in $U_2$, who respond based on $\hat \mu_1$ and their own data. The analyst then aggregates this second set of responses into a better final estimate of $\mu$.

\begin{theorem}
\label{thm:kv_two}
	Two-round protocol \algok~satisfies $(\eps,0)$-local differential privacy for $x_1, \ldots, x_n$ and, if $x_1, \ldots, x_n \sim_{iid} N(\mu,\sigma^2)$ where $\sigma$ is known and $\tfrac{n}{\log(n)} = \Omega\left(\tfrac{\log(\mu)\log(1/\beta)}{\eps^2}\right)$, with probability $1-\beta$ outputs $\hat \mu$ such that $|\hat \mu - \mu| = O\left(\tfrac{\sigma}{\eps}\sqrt{\tfrac{\log\left(1/\beta\right)}{n}}\right)$.
\end{theorem}

\begin{algorithm}[h]
	\caption{\algok}
	\begin{algorithmic}[1]
		\REQUIRE $\eps, k, \L, n, \sigma, U_1, U_2$
		\FOR{$j \in \L$}
			\FOR{user $i \in U_1^j$}
				\STATE User $i$ outputs $\tilde y_i \gets \rrone(\eps,i,j)$
			\ENDFOR
		\ENDFOR \COMMENT{End of round 1}
		\STATE Analyst computes $\hat H_1 \gets	\aggonek(\eps, k, \L, U_1)$
		\STATE Analyst computes $\hat \mu_1 \gets \hatmuone(\beta, \eps, \hat H_1, k, \L)$ 
		\FOR{user $i \in U_2$}
			\STATE User $i$ outputs
			$\tilde y_i \gets \rrtwok(\eps, i, \hat \mu_1, \sigma)$
		\ENDFOR \COMMENT{End of round 2}
		\STATE Analyst computes	$\hat H_2 \gets \aggtwok(\eps, n/2, U_2)$
		\STATE Analyst computes
		$\hat T \gets
		\sqrt{2} \cdot \erf^{-1}\left(\frac{2(-\hat H_2(-1) + \hat H_2(1))}{n}\right)$
		\STATE Analyst outputs $\hat \mu_2 \gets \sigma \hat T + \hat \mu_1$
		\ENSURE Analyst estimate $\hat \mu_2$ of $\mu$
	\end{algorithmic}
\end{algorithm}

\subsubsection{First round of~\algok}
For neatness, let $L = \lfloor n/(2k) \rfloor$, $\lmin = \lfloor \log(\sigma) \rfloor$, $\lmax = \lmin-1 + L$, and $\L = \{\lmin, \lmin+1, \ldots, \lmax\}$. $U_1$ is then split into $L$ subgroups indexed by $\L$, and each subgroup has size $k = \Omega\left(\tfrac{\log(n/\beta)}{\eps^2}\right)$. \algok~begins by iterating through each subgroup $j \in \L$. Each user $i \in U_1^j$ releases a privatized version of $\lfloor x_i/2^j \rfloor \mod 4$ via randomized response (\rrone): with probability $e^\eps/(e^\eps+3)$, user $i$ outputs $\lfloor x_i/2^j \rfloor \mod 4$, and otherwise outputs one of the remaining elements of $\{0,1,2,3\}$ uniformly at random. Responses from group $U_1^j$ will be used to estimate the $j^{th}$ least significant bit of $\mu$ (rounded to an integer). The analyst then uses \aggonek~(``Known Variance Aggregation'') to aggregate and debias responses to account for this randomness.
 
 \begin{algorithm}[H]
	\caption{\aggonek}
	\begin{algorithmic}[1]
		\REQUIRE $\eps, k, \L, U$
		\FOR{$j \in \L$}
			\FOR{$a \in \{0,1\}$}
				\STATE $C^j(a) \gets |\{\tilde y_i \mid i \in U^j, \tilde y_i = a\}|$
				\STATE $\hat H^j(a) \gets \frac{e^\eps+3}{e^\eps-1} \cdot
				\left(C^j(a) - \frac{k}{e^\eps+3}\right)$
			\ENDFOR		
		\ENDFOR
		\STATE Output $\hat H$
		\ENSURE Aggregated histogram $\hat H$ of private user responses
		\end{algorithmic}
\end{algorithm}

The result is a collection of histograms $\hat H_1$. The analyst uses $\hat H_1$ in \hatmuone~to binary search for $\mu$. Intuitively, for each subgroup $U_1^j$ if all multiples of $2^j$ are far from $\mu$ then Gaussian concentration implies that almost all users $i \in U_1^j$ compute the same value of $\lfloor x/2^j \rfloor \mod 4$. This produces a histogram $\hat H_1^j$ where most elements fall concentrate in a single bin. The analyst in turn narrows their search range for $\mu$. For example, if $\hat H_1^{\lmax}$ concentrates in $0$, then the range narrows to $\mu \in [0,2^{\lmax})$; if $\hat H_1^{\lmax - 1}$ concentrates in $1$, then the range narrows to $\mu \in [2^{\lmax - 1}, 2^{\lmax})$, and so on.

If instead some multiple of $2^j$ is near $\mu$, the elements of $\hat H_1^j$ will spread over multiple (adjacent) bins. This is also useful: a point from the ``middle'' of this block of bins is $O(\sigma)$-close to $\mu$. The analyst thus takes such a point as $\hat \mu_1$ and ends their search. Our analysis will also rely on having a noticeably low-count bin that is non-adjacent to the bin containing $\mu$. This motivates using 4 as a modulus.

In this way, the analyst examines $\hat H_1^{\lmax}, \hat H_1^{\lmax-1}, \ldots$ in sequence,  estimating $\mu$ from most to least significant bit. Crucially, the modulus structure of user responses enables the analyst to carry out this binary search with \emph{one} round of interaction. Thus at the end of the first round the analyst obtains an $O(\sigma)$-accurate estimate $\hat \mu_1$ of $\mu$.

\begin{algorithm}[h]
	\caption{\hatmuone}
	\begin{algorithmic}[1]
		\REQUIRE $\beta, \eps, \hat H_1, k, \L$
		\STATE $\psi \gets
		\left(\tfrac{\eps+4}{\eps\sqrt{2}}\right) \cdot \sqrt{k\ln(8L/\beta)}$
		\STATE $j \gets \lmax$ 
		\STATE $I_j \gets [0, 2^{\lmax}]$
		\WHILE{$j \geq \lmin$ and $\max_{a \in \{0,1,2,3\}} \hat H_1^j(a)
		\geq 0.52k + \psi$}
			\STATE Analyst computes integer $c$ such that $c2^j \in I_j$ and $c \equiv M_1(j) \mod 4$
			\STATE Analyst computes $I_{j-1} \gets [c2^j, (c+1)2^j]$
			\STATE $j \gets j-1$
		\ENDWHILE
		\STATE $j \gets \max(j,\lmin)$
		\STATE Analyst computes $M_1(j) \gets \arg \max_{a \in \{0,1,2,3\}} \hat H_1^j(a)$
		\STATE Analyst computes $M_2(j) \gets \arg \max_{a \in \{0,1,2,3\} - \{M_1(j)\}} \hat H_1^j(a)$
		\STATE Analyst computes $c^* \gets$ maximum integer such that $c^*2^j \in I_j$ and $c^* \equiv M_1(j)$ or $M_2(j) \mod 4$
		\STATE Analyst outputs $\hat \mu_1 \gets c^*2^j$
		\ENSURE Initial estimate $\hat \mu_1$ of $\mu$
		\end{algorithmic}
\end{algorithm}

\subsubsection{Second round of~\algok}
In the second round, the analyst passes $\hat \mu_1$ to users in $U_2$. Users respond through \rrtwok~(``Known Variance Randomized Response''), a privatized version of an algorithm from the distributed statistical estimation literature~\citep{BGMNW16}. In \rrtwok , each user centers their point with $\hat \mu_1$, standardizes it using $\sigma$, and randomized responds on $\sgn((x_i - \hat \mu_1)/\sigma)$.  This crucially relies on the first estimate $\hat \mu_1$, as properly centering requires an initial $O(\sigma)$-accurate estimate of $\hat \mu$. The analyst can then aggregate these responses by a debiasing process \aggtwok~akin to \aggonek .

\begin{algorithm}[h]
	\caption{\rrtwok}
	\begin{algorithmic}[1]
		\REQUIRE $\eps, i, \hat \mu_1, \sigma$
		\STATE User $i$ computes $x_i' \gets (x_i - \hat \mu_1)/\sigma$
		\STATE User $i$ computes $y_i \gets \sgn(x_i')$
		\STATE User $i$ computes $c \sim_U [0,1]$
		\IF{$c \leq \tfrac{e^\eps}{e^\eps+1}$}
			\STATE User $i$ publishes $\tilde y_i \gets y_i$
		\ELSE
			\STATE User $i$ publishes $\tilde y_i \gets -y_i$
		\ENDIF
		\ENSURE Private centered user estimate $\tilde y_i$ 
	\end{algorithmic}
\end{algorithm}

\begin{algorithm}[h]
	\caption{\aggtwok}
	\begin{algorithmic}[1]
		\REQUIRE $\eps, k, U$
		\FOR{$a \in \{-1,1\}$}
				\STATE $C(a) \gets |\{\tilde y_i \mid i \in U, \tilde y_i = a\}|$ 
				\STATE $\hat H(a) \gets \frac{e^\eps+1}{e^\eps-1} \cdot
				\left(C(a) - \frac{k}{e^\eps+1}\right)$
		\ENDFOR
		\STATE Analyst outputs $\hat H$
		\ENSURE Aggregated histogram $\hat H$ of private user responses
		\end{algorithmic}
\end{algorithm}

From this aggregation $\hat H_2$, the analyst obtains a good estimate of the bias of the initial estimate $\hat \mu_1$. If $\hat \mu_1 < \mu$, responses will skew toward 1, and if $\hat \mu_1 > \mu$ responses will skew toward $-1$. By comparing this skew to the true standard CDF using the error function \erf , the analyst recovers a better final estimate $\hat \mu_2$ of $\mu$ (Lines 12-13 of \algok). Privacy of \algok~follows from the privacy of the randomized response mechanisms \rrone~and \rrtwok.

\subsection{One-round Protocol~\algokone}
\label{subsec:kv_one}
Recall that in \algok~the analyst 1) employs user pool $U_1$ to  compute rough estimate $\hat \mu_1$ and 2) adaptively refines this estimate using responses from the second user pool $U_2$.  \algokone~executes these two rounds of \algok~\emph{simultaneously} by parallelization. 

\begin{theorem}
\label{thm:kv_one}
	One-round protocol \algokone~satisfies $(\eps,0)$-local differential privacy for $x_1, \ldots, x_n$ and, if $x_1, \ldots, x_n \sim_{iid} N(\mu,\sigma^2)$ where $\sigma$ is known and $\tfrac{n}{\log(n)} = \Omega\left(\frac{\log(\mu)\log(1/\beta)}{\eps^2}\right)$, with probability $1-\beta$ outputs $\hat \mu$ such that $$|\hat \mu - \mu| = O\left(\tfrac{\sigma}{\eps}\sqrt{\tfrac{\log\left(1/\beta\right)\sqrt{\log(n)}}{n}}\right).$$
\end{theorem}

\algokone~splits $U_2$ into $\Theta(\sqrt{\log(n)})$ subgroups that run the second-round protocol from \algok~with different values of $\hat \mu_1$. Intuitively, it suffices that at least one subgroup centers using a $\hat \mu_1$ near $\mu$: the analyst can then use the data from that subgroup and discard the rest. By Gaussian concentration, most user samples cluster within $O(\sigma\sqrt{\log(n)})$ of $\mu$, so each subgroup $U_2^j$ receives a set of points $S(j)$ interspersed $\Theta(\sigma\sqrt{\log(n)})$ apart on the real line, and each user $i \in U_2^j$ centers using the point in $S(j)$ closest to $x_i$. This leads us to use $\Theta(\sqrt{\log(n)})$ groups with each point in $S(j+1)$ shifted $\Theta(\sigma)$ from the corresponding point in $S(j)$. By doing so, we ensure that some subgroup has most of its users center using a point within $O(\sigma)$ of $\mu$.

In summary, \algokone~works as follows: after collecting the single round of responses from $U_1$ and $U_2$, the analyst computes $\hat \mu_1$ using responses from $U_1$. By comparing $\hat \mu_1$ and $S(j)$ for each $j$, the analyst then selects the subgroup $U_2^{j^*}$ where most users centered using a value in $S(j^*)$ closest to $\hat \mu_1$. This mimics the effect of adaptively passing $\hat \mu_1$ to the users in $U_2^{j^*}$, so the analyst simply processes the responses from $U_2^{j^*}$ as it processed responses from $U_2$ in \algok. Because $U_2^{j^*}$ contains $\Theta(n/\sqrt{\log(n)})$ users, the cost is a $\log^{1/4}(n)$ factor in accuracy.
\section{Unknown Variance}
\label{sec:upper_unknown}
In this section, we consider the more general problem with unknown variance $\sigma^2$ (shorthanded ``UV'') that lies in known interval $[\sigma_{\min},\sigma_{\max}]$. We again provide a two-round protocol \algo~and a slightly less accurate one-round protocol \algoone.

\subsection{Two-round Protocol}
\label{subsec:uv_two}
\algo~is structurally similar to \algok . In round one, the analyst uses the responses of half of the users to roughly estimate $\mu$, and in round two the analyst passes this estimate to the second half of users for improvement. However, two key differences now arise. First, since $\sigma$ is unknown, the analyst must now also estimate $\sigma$ in round one. Second, since the analyst does not have a very accurate estimate of $\sigma$, the refinement process of the second round employs Laplace noise rather than the CDF comparison used in~\algok .

\begin{theorem}
\label{thm:uv_two}
	Two-round protocol \algo~satisfies $(\eps,0)$-local differential privacy for $x_1, \ldots, x_n$ and, if $x_1, \ldots, x_n \sim_{iid} N(\mu,\sigma^2)$ where $\sigma$ is unknown but bounded in known $[\sigma_{\min},\sigma_{\max}]$ and $\tfrac{n}{\log(n)} = \Omega\left(\frac{\left[\log\left(\frac{\sigma_{\max}}{\sigma_{\min}}+1\right) + \log(\mu)\right]\log\left(\frac{1}{\beta}\right)}{\eps^2}\right)$, with probability at least $1-\beta$ outputs $\hat \mu$ such that $$|\hat \mu - \mu| = O\left(\tfrac{\sigma}{\eps}
\sqrt{\tfrac{\log\left(1/\beta\right)\log(n)}{n}}\right).$$
\end{theorem}

\begin{algorithm}[h]
	\caption{\algo}
	\begin{algorithmic}[1]
		\REQUIRE $\eps, k_1, \L_1, n, \sigma, U_1, U_2$
		\FOR{$j \in \L_1$}
			\FOR{user $i \in U_1^j$}
				\STATE User $i$ outputs $\tilde y_i \gets \rrone(\eps,i,j)$
			\ENDFOR
		\ENDFOR \COMMENT{End of round 1}
		\STATE Analyst computes $\hat H_1 \gets \aggone(\eps, \L_1, U_1)$
		\STATE Analyst computes
		$\hat \sigma \gets \hatsigma(\beta, \eps, \hat H_1, k_1, \L_1)$
		\STATE Analyst computes $\hat H_2 \gets \aggonek(\eps, k_1, \L_1, U_1)$
		\STATE Analyst computes $\hat \mu_1 \gets \hatmuone(\beta, \eps, \hat H_2, k_1, \L_1)$
		\STATE Analyst computes	$I \gets [\hat \mu_1 \pm \hat \sigma(2 + \sqrt{\ln(4n)})]$	
		\FOR{user $i \in U_2$}
			\STATE User $i$ outputs $\tilde y_i \gets \rrtwo(\eps, i, I)$
		\ENDFOR \COMMENT{End of round 2}
		\STATE Analyst outputs $\hat \mu_2 \gets \tfrac{2}{n}\sum_{i \in U_2} \tilde y_i$ 
		\ENSURE Analyst estimate $\hat \mu_2$ of $\mu$
	\end{algorithmic}
\end{algorithm}

\subsubsection{First round of \algo}
Similarly to~\algok , we split $U_1$ into $L_1 = \lfloor n/(2k_1) \rfloor$ subgroups of size $k_1 = \Omega\left(\tfrac{\log(n/\beta)}{\eps^2}\right)$ and define $\lmin = \lfloor \log(\sigma_{\min}) \rfloor$, $\lmax = L_1 + \lmin - 1 \geq \lceil \log(\sigma_{\max}) \rceil$, and $\L_1 = \{\lmin, \lmin+1, \ldots, \lmax\}$, indexing $U_1$ by $\L_1$. 

Also as in \algok , each user $i$ in each subgroup $U_1^j$ publishes a privatized version of $\lfloor x_i/2^j \rfloor \mod 4$. The analyst aggregates them (\aggonek ) into $\hat H_2$ and roughly estimates $\mu$ (\hatmuone ) as in~\algok. However, the analyst also employs a (similar) aggregation (\aggone ) into $\hat H_1$ for estimating $\sigma$ (\hatsigma ). At a high level, because samples from $N(\mu, \sigma^2)$ probably  fall within $3\sigma$ of $\mu$, when $2^j \gg \sigma$ there exist $a, a+1 \mod 4 \in \{0,1,2,3\}$ such that almost all users $i$ have $\lfloor x_i/2^j \rfloor \mod 4 \in \{a,a+1\}$.  The analyst's debiased aggregated histogram $\hat H_1^j$ thus concentrates in at most two adjacent bins when $2^j \gg \sigma$ and spreads over more bins when $2^j \ll \sigma$. By a process like~\hatmuone, examining this transition from concentrated to unconcentrated in $\hat H_1^{\lmax}, \hat H_1^{\lmax-1}, \ldots$ yields a rough estimate of when $2^j \gg \sigma$ versus when $2^j \ll \sigma$. As a result, at the end of round one the analyst obtains $O(\sigma)$-accurate estimates $\hat \sigma$ of $\sigma$ and $\hat \mu_1$ of $\mu$.

\subsubsection{Second round of~\algo}
The analyst now refines their initial estimate of $\mu$. First, the analyst constructs an interval $I$ of size $O(\hat \sigma\sqrt{\log(n)})$ around $\hat \mu_1$. Users in $U_2$ then truncate their values to $I$, add Laplace noise scaled to $|I|$ (the sensitivity of releasing a truncated point), and send the result to the analyst using \rrtwo. The analyst then simply takes the mean of these responses as the final estimate of $\mu$. Its accuracy guarantee follows from concentration of user samples around $\mu$ and Laplace noise around 0. Privacy follows from our use of randomized response and Laplace noise.

We briefly explain our use of Laplace noise rather than CDF comparison. Roughly, when using an estimate $\hat \sigma$ in the centering process, error in $\hat \sigma$ propagates to error in the final estimate $\hat \mu_2$. This leads us to Laplace noise, which better handles the error in $\hat \sigma$ that estimation of $\sigma$ introduces. The cost is the $\sqrt{\log(n)}$ factor that arises from adding Laplace noise scaled to $|I|$. Our choice of $|I|$ --- constructed to contain not only $\mu$ but the points of $\Omega(n)$ users --- thus strikes a deliberate balance. $I$ is both large enough to cover most users (who would otherwise truncate too much and skew the responses) and small enough to not introduce much noise from privacy (as noise is scaled to $\Lap{|I|/\eps}$).

\subsection{One-round Protocol}
\label{subsec:uv_one}
We now provide a one-round version of \algo, \algoone .

\begin{theorem}
\label{thm:uv_one}
	One-round protocol \algoone~satisfies $(\eps,0)$-local differential privacy for $x_1, \ldots, x_n$ and, if $x_1, \ldots, x_n \sim_{iid} N(\mu,\sigma^2)$ where $\sigma$ is unknown but bounded in known $[\sigma_{\min},\sigma_{\max}]$ and $\tfrac{n}{\log(n)} = \Omega\left(\frac{\left[\log\left(\frac{\sigma_{\max}}{\sigma_{\min}}+1\right) + \log(\mu)\right]\log\left(\frac{1}{\beta}\right)}{\eps^2}\right)$, with probability at least $1-\beta$ outputs $\hat \mu$ with $$|\hat \mu - \mu| = O\left(\tfrac{\sigma}{\eps}	\sqrt{\tfrac{\log\left(\tfrac{\sigma_{\max}}{\sigma_{\min}} + 1\right)\log\left(1/\beta\right)\log^{3/2}(n)}{n}}\right).$$
\end{theorem}

Like \algokone, \algoone~simulates the second round of \algo~simultaneously with its first round. \algoone~splits $U_2$ into subgroups, where each subgroup responds using a \emph{different} interval $I_j$. At the end of the single round the analyst obtains estimates $\hat \mu_1$ and $\hat \sigma$ from users in $U_1$, constructs an interval $I$ from these estimates, and finds a subgroup of $U_2$ where most users employed a similar interval $I_j$. This similarity guarantees that the subgroup's responses yield the same accuracy as the two-round case up to an $O(\text{\# subgroups})$ factor. As in \algokone , we rely on Gaussian concentration and the modulus trick to minimize the number of subgroups. However, this time we parallelize not only over possible values of $\hat \mu_1$ but possible values of $\hat \sigma$ as well. As this parallelization is somewhat involved, we defer its presentation to the Supplement.

In summary, at the end of the round the analyst computes $\hat \mu_1$ and $\hat \sigma$, computes the resulting interval $I^*$, and identifies a subgroup of $U_2$ that responded using an interval $I_j$ similar to $I^*$. This mimics the effect of passing an interval of size $O(\sigma\sqrt{\log(n)})$ around $\hat \mu_1$ to this subgroup and using the truncate-then-Laplace noise method of \algo . The cost, due to the $g = O\left(\left[\log\left(\tfrac{\sigma_{\max}}{\sigma_{\min}}\right) + 1\right]\sqrt{\log(n)}\right)$ subgroups required, is the $1/\sqrt{g}$ reduction in accuracy shown in Theorem~\ref{thm:uv_one}.

\section{Lower Bound}
\label{sec:lower}
We now show that all of our upper bounds are tight up to logarithmic factors. Our argument has three steps: we first reduce our estimation problem to a testing problem, then reduce this testing problem to a purely locally private testing problem, and finally prove a lower bound for this purely locally private testing problem. Taken together, these results show that estimation is hard for sequentially interactive $(\eps,\delta)$-locally private protocols. An extension to fully interactive protocols using recent subsequent work by~\citet{JMNR19} appears in the Supplement.

\begin{theorem}
\label{thm:lb_formal}
	Let $\delta < \min\left(\frac{\epsilon\beta}{60n\ln(5n/2\beta)}, \frac{\beta}{16n\ln(n/\beta)e^{7\eps}}\right)$ and $\eps > 0$. There exists absolute constant $c$ such that if $\A$ is an $(\eps,\delta)$-locally private $(\alpha,\beta)$-estimator for $\est{n,M,\sigma}$ where $M = \sigma/[4(e^\eps-1)\sqrt{2nc}]$ and $\beta <1/16$, then $\alpha \geq M/2 = \Omega\left(\frac{\sigma}{\eps}\sqrt{\frac{1}{n}}\right)$.
\end{theorem}

\bibliographystyle{plainnat}
\bibliography{gaussian_local}

\begin{thebibliography}{26}
\providecommand{\natexlab}[1]{#1}
\providecommand{\url}[1]{\texttt{#1}}
\expandafter\ifx\csname urlstyle\endcsname\relax
  \providecommand{\doi}[1]{doi: #1}\else
  \providecommand{\doi}{doi: \begingroup \urlstyle{rm}\Url}\fi

\bibitem[Abowd(2016)]{A16}
John~M. Abowd.
\newblock The challenge of scientific reproducibility and privacy protection
  for statistical agencies.
\newblock Technical report, Census Scientific Advisory Committee, 2016.

\bibitem[Apple(2017)]{A17}
Differential Privacy~Team Apple.
\newblock Learning with privacy at scale.
\newblock Technical report, Apple, 2017.

\bibitem[Avent et~al.(2017)Avent, Korolova, Zeber, Hovden, and
  Livshits]{AKZHL17}
Brendan Avent, Aleksandra Korolova, David Zeber, Torgeir Hovden, and Benjamin
  Livshits.
\newblock Blender: enabling local search with a hybrid differential privacy
  model.
\newblock In \emph{USENIX Security Symposium}, 2017.

\bibitem[Beimel et~al.(2008)Beimel, Nissim, and Omri]{BNO08}
Amos Beimel, Kobbi Nissim, and Eran Omri.
\newblock Distributed private data analysis: Simultaneously solving how and
  what.
\newblock In \emph{International Cryptology Conference (CRYPTO)}, 2008.

\bibitem[Bittau et~al.(2017)Bittau, Erlingsson, Maniatis, Mironov, Raghunathan,
  Lie, Rudominer, Kode, Tinnes, and Seefeld]{BEMMR+17}
Andrea Bittau, \'{U}lfar Erlingsson, Petros Maniatis, Ilya Mironov, Ananth
  Raghunathan, David Lie, Mitch Rudominer, Ushasree Kode, Julien Tinnes, and
  Bernhard Seefeld.
\newblock Prochlo: Strong privacy for analytics in the crowd.
\newblock In \emph{Symposium on Operating Systems Principles (SOSP)}, 2017.

\bibitem[Braverman et~al.(2016)Braverman, Garg, Ma, Nguyen, and
  Woodruff]{BGMNW16}
Mark Braverman, Ankit Garg, Tengyu Ma, Huy~L Nguyen, and David~P Woodruff.
\newblock Communication lower bounds for statistical estimation problems via a
  distributed data processing inequality.
\newblock In \emph{Symposium on the Theory of Computing (STOC)}, 2016.

\bibitem[Bun et~al.(2018)Bun, Nelson, and Stemmer]{BNS18}
Mark Bun, Jelani Nelson, and Uri Stemmer.
\newblock Heavy hitters and the structure of local privacy.
\newblock In \emph{Symposium on Principles of Database Systems (PODS)}, 2018.

\bibitem[Chan et~al.(2011)Chan, Shi, and Song]{CSS11}
T.-H.~Hubert Chan, Elaine Shi, and Dawn Song.
\newblock Private and continual release of statistics.
\newblock \emph{ACM Trans. Inf. Syst. Secur.}, 2011.

\bibitem[Cheu et~al.(2019)Cheu, Smith, Ullman, Zeber, and Zhilyaev]{CSUZZ18}
Albert Cheu, Adam Smith, Jonathan Ullman, David Zeber, and Maxim Zhilyaev.
\newblock Distributed differential privacy via mixnets.
\newblock In \emph{International Conference on Theory and Application of
  Cryptographic Techniques (EUROCRYPT)}, 2019.

\bibitem[Daniely and Feldman(2019)]{DF18}
Amit Daniely and Vitaly Feldman.
\newblock Learning without interaction requires separation.
\newblock In \emph{Neural Information and Processing Systems (NeurIPS)}, 2019.

\bibitem[Ding et~al.(2017)Ding, Kulkarni, and Yekhanin]{DKY17}
Bolin Ding, Janardhan Kulkarni, and Sergey Yekhanin.
\newblock Collecting telemetry data privately.
\newblock In \emph{Neural Information Processing Systems (NIPS)}, 2017.

\bibitem[Duchi and Rogers(2019)]{DR19}
John Duchi and Ryan Rogers.
\newblock Lower bounds for locally private estimation via communication
  complexity.
\newblock In \emph{Conference on Learning Theory (COLT)}, 2019.

\bibitem[Duchi et~al.(2013)Duchi, Jordan, and Wainwright]{DJW13}
John~C Duchi, Michael~I Jordan, and Martin~J Wainwright.
\newblock Local privacy and statistical minimax rates.
\newblock In \emph{Foundations of Computer Science (FOCS)}, 2013.

\bibitem[Dwork et~al.(2006)Dwork, McSherry, Nissim, and Smith]{DMNS06}
Cynthia Dwork, Frank McSherry, Kobbi Nissim, and Adam Smith.
\newblock Calibrating noise to sensitivity in private data analysis.
\newblock In \emph{Theory of Cryptography Conference (TCC)}, 2006.

\bibitem[Dwork et~al.(2014)Dwork, Roth, et~al.]{DR14}
Cynthia Dwork, Aaron Roth, et~al.
\newblock The algorithmic foundations of differential privacy.
\newblock \emph{Foundations and Trends{\textregistered} in Theoretical Computer
  Science}, 2014.

\bibitem[Erlingsson et~al.(2014)Erlingsson, Pihur, and Korolova]{EPK14}
{\'U}lfar Erlingsson, Vasyl Pihur, and Aleksandra Korolova.
\newblock Rappor: Randomized aggregatable privacy-preserving ordinal response.
\newblock In \emph{Conference on Computer and Communications Security (CCS)},
  2014.

\bibitem[Gaboardi et~al.(2019)Gaboardi, Rogers, and Sheffet]{GRS18}
Marco Gaboardi, Ryan Rogers, and Or~Sheffet.
\newblock Locally private mean estimation: Z-test and tight confidence
  intervals.
\newblock In \emph{International Conference on Artificial Intelligence and
  Statistics (AISTATS)}, 2019.

\bibitem[Joseph et~al.(2019)Joseph, Mao, Neel, and Roth]{JMNR19}
Matthew Joseph, Jieming Mao, Seth Neel, and Aaron Roth.
\newblock The role of interactivity in local differential privacy.
\newblock In \emph{Foundations of Computer Science (FOCS)}, 2019.

\bibitem[Kamath et~al.(2019)Kamath, Li, Singhal, and Ullman]{KLSU18}
Gautam Kamath, Jerry Li, Vikrant Singhal, and Jonathan Ullman.
\newblock Privately learning high-dimensional distributions.
\newblock In \emph{Conference on Learning Theory (COLT)}, 2019.

\bibitem[Karwa and Vadhan(2018)]{KV17}
Vishesh Karwa and Salil Vadhan.
\newblock {Finite Sample Differentially Private Confidence Intervals}.
\newblock In \emph{Innovations in Theoretical Computer Science Conference
  (ITCS)}, 2018.

\bibitem[Kasiviswanathan et~al.(2011)Kasiviswanathan, Lee, Nissim,
  Raskhodnikova, and Smith]{KLNRS11}
Shiva~Prasad Kasiviswanathan, Homin~K Lee, Kobbi Nissim, Sofya Raskhodnikova,
  and Adam Smith.
\newblock What can we learn privately?
\newblock \emph{SIAM Journal on Computing}, 2011.

\bibitem[Kuo et~al.(2018)Kuo, Chiu, Kifer, Hay, and Machanavajjhala]{KCKHM18}
Yu-Hsuan Kuo, Cho-Chun Chiu, Daniel Kifer, Michael Hay, and Ashwin
  Machanavajjhala.
\newblock Differentially private hierarchical count-of-counts histograms.
\newblock In \emph{International Conference on Very Large Databases (VLDB)},
  2018.

\bibitem[Raginsky(2016)]{R16}
Maxim Raginsky.
\newblock Strong data processing inequalities and $\phi$-sobolev inequalities
  for discrete channels.
\newblock \emph{IEEE Transactions on Information Theory}, 62\penalty0
  (6):\penalty0 3355--3389, 2016.

\bibitem[Smith et~al.(2017)Smith, Thakurta, and Upadhyay]{STU17}
Adam Smith, Abhradeep Thakurta, and Jalaj Upadhyay.
\newblock Is interaction necessary for distributed private learning?
\newblock In \emph{Symposium on Security and Privacy (SP)}, 2017.

\bibitem[Ullman(2018)]{U18}
Jonathan Ullman.
\newblock Tight lower bounds for locally differentially private selection.
\newblock \emph{arXiv preprint arXiv:1802.02638}, 2018.

\bibitem[Vadhan(2017)]{V17}
Salil Vadhan.
\newblock The complexity of differential privacy.
\newblock In \emph{Tutorials on the Foundations of Cryptography}, pages
  347--450. Springer, 2017.

\end{thebibliography}

\section{Proofs from Section~\ref{subsec:kv_two}}
\label{subsec:kv_two_supp}
We start with pseudocode for \rrone .

\begin{algorithm}[H]
	\caption{\rrone}
	\begin{algorithmic}[1]
		\REQUIRE $\eps, i, j$
		\STATE $y_i \gets \lfloor x_i/2^j \rfloor \mod 4$
		\IF{$p \sim_U [0,1] \leq \tfrac{e^\eps}{e^\eps+3}$}
			\STATE User $i$ publishes $\tilde y_i \gets y_i$
		\ELSE
			\STATE User $i$ publishes $\tilde y_i \sim_u (\{0,1,2,3\} \backslash \{y_i\})$
		\ENDIF
		\ENSURE Private user estimate $\tilde y_i$ of $\mu(j)$
		\end{algorithmic}
\end{algorithm}

Next, we prove the privacy guarantee for \algok .

\begin{theorem}
\label{thm:algok_privacy}
	\algok~satisfies $(\eps,0)$-local differential privacy for $x_1, \ldots, x_n$.
\end{theorem}
\begin{proof}
	As \algok~is sequentially interactive, each user only produces one output. It therefore suffices to show that each randomized response routine used in \algok~is $(\eps,0)$-locally private. In \rrone, for any possible inputs $x, x'$ and output $y$ we have
	\[
		\frac{\P{\rrone(x) =y}}{\P{\rrone(x') = y}} \leq
		\frac{e^\eps/(e^\eps+3)}{1/(e^\eps+3)} \leq e^\eps
	\]
	so \rrone~is $(\eps,0)$-locally private. \rrtwok~is $(\eps,0)$-locally private by similar logic. 
\end{proof}

We now prove the accuracy guarantee for \algok . First, recall that $\hat H_1$ is the aggregation (via \aggonek) of user responses (via \rrone). Let $H_1$ be the ``true'' histogram, $H_1^j(a) = |\{y_i \mid i \in U_1^j, y_i = a\}|$ for all $a \in \{0,1,2,3\}$ and $j \in \L$. Since the analyst only has access to $\hat H_1$, we need to show that $\hat H_1$ and $H_1$ are similar.

\begin{lemma}
\label{lem:hhk}
	With probability at least $1-\beta$, for all $j \in \L$,
	\[
		||\hat H_1^j - H_1^j||_\infty \leq
		\left(\tfrac{\eps+4}{\eps\sqrt{2}}\right) \cdot \sqrt{k\ln(8L/\beta)}.
	\]
\end{lemma}
\begin{proof}
	Choose $a \in \{0,1,2,3\}$ and $j \in \L$. 
	$\E{C^j(a)} = \tfrac{H_1^j(a)e^\eps}{e^\eps+3} + \tfrac{k-H_1^j(a)}{e^\eps+3}
	= \tfrac{H_1^j(a)(e^\eps-1) + k}{e^\eps+3}$, so by a pair of
	Chernoff bounds on the $k$ users in $U_1^j$, with probability at least
	$1-\beta/4L$,
	\[
		|C^j(a) - \tfrac{H_1^j(a)(e^\eps-1) + k}{e^\eps+3}|
		\leq \sqrt{k\ln(8L/\beta)/2}.
	\]
	Then since $\hat H_1^j(a) = \frac{e^\eps+3}{e^\eps-1} \cdot
	\left(C^j(a) - \frac{k}{e^\eps+3}\right)$, this implies
	\[
		|\hat H_1^j(a) - H_1^j(a)| \leq \frac{e^\eps+3}{e^\eps-1} \cdot
		\sqrt{k\ln(8L/\beta)/2} <
		\left(\tfrac{\eps+4}{\eps\sqrt{2}}\right) \cdot \sqrt{k\ln(8L/\beta)}
	\]
	where the last step uses $\tfrac{e^\eps+3}{e^\eps-1} < \tfrac{\eps+4}{\eps}$.
	Union bounding over $a \in \{0,1,2,3\}$ and all $L$ groups $U_1^j$ completes the
	proof.
\end{proof}

Next, we show how the analyst uses $\hat H_1$ to estimate $\mu$ through \hatmuone. Intuitively, in subgroup $U_1^j$ when user responses concentrate in a single bin $\mod 4$, this suggests that $\mu$ lies in the corresponding bin. In the other direction, when user responses do not concentrate in a single bin, users with points near $\mu$ must spread out over multiple bins, suggesting that $\mu$ lies near the boundary between bins. We formalize this intuition in \hatmuone~and Lemma~\ref{lem:hatmuonek}.

\begin{lemma}
\label{lem:hatmuonek}
	Conditioned on the success of the preceding lemmas, with probability at
	least $1-\beta$, $|\hat \mu_1 - \mu| \leq 2\sigma$.
\end{lemma}
\begin{proof}
	Recall the definitions of $\psi$, $M_1(j)$, and $M_2(j)$ from the pseudocode
	for \hatmuone:
	$\psi = \left(\tfrac{\eps+4}{\eps\sqrt{2}}\right) \cdot \sqrt{k\ln(8L/\beta)}$,
	$M_1(j) = \arg \max_{a \in \{0,1,2,3\}} \hat H_1^j(a)$, and
	$M_2(j) = \arg \max_{a \in \{0,1,2,3\} - \{M_1(j)\}} \hat H_1^j(a)$. We start
	by proving two useful claims.
	
	\underline{Claim 1}: With probability at least $1-\beta/5$, for all
	$j \in \L$ where $2^j > \sigma$, if $j' = \lmax, \lmax-1, \ldots, j+1$ all
	have $\hat H_1^{j'}(M_1(j)) \geq 0.52k + \psi$, then $\mu \in I_j$.
	
	To see why, suppose $2^j > \sigma$ and let $x \sim N(\mu,\sigma^2)$. Recall
	the Gaussian CDF $F(x) = \tfrac{1}{2}
	\left[1 + \text{erf}\left(\tfrac{x-\mu}{\sigma\sqrt{2}}\right)\right]$. Then
	for any $a \not \equiv \lfloor \mu/2^j \rfloor \mod 4$
	\[
		\P{\lfloor x/2^j \rfloor \equiv a \mod 4}
		\leq \P{x \not \in [\mu, \mu+3 \cdot 2^j)}
		< \P{x \not \in [\mu,\mu+3\sigma)} < 0.51
	\]
	where the second inequality uses $2^j > \sigma$. Thus by a binomial
	Chernoff bound, the assumption $k > 5000\ln(5L/\beta)$, and
	Lemma~\ref{lem:hhk}, with probability $\geq 1-\beta/5L$,
	$\hat H_1^j(a) < 0.52k + \psi$. Therefore if for some $a$ we have 
	$\hat H_1^j(a) \geq 0.52k + \psi$, $a \equiv \lfloor \mu/2^j \rfloor \mod 4$.
	Moreover, if $\mu \in I_j$ then letting $c$ be the (unique) integer such that
	$c \equiv M_1(j) \mod 4$ and $c2^j \in I_j$ (since $I_j$ has endpoints 
	$c_12^j$ and $(c_1+2)2^j$ for integer $c_1$) we get
	$\mu \in [c2^j, (c+1)2^j] = I_j$. As $\mu \in I_{\lmax}$
	by our assumed lower bound on $n$, the claim follows by induction.
	
	\underline{Claim 2}: Let $j$ be the maximum $j \in \L$ with
	$\hat H_1^j(M_1(j)) < 0.52k + \psi$, and let $c^*$ be the maximum integer such
	that $c^*2^j \in I_j$ and $c^* \equiv M_1(j)$ or $M_2(j) \mod 4$. If
	$2^j > \sigma$, then with probability at least $1-4\beta/5$,
	$|c^*2^j - \mu| \leq 2\sigma$.
	
	To see why, first note that by Claim 1, $\mu \in I_j$. Let $[c2^j, (c+1)2^j)$
	be the subinterval of $I_j$ containing $\mu$ for integer $c$. Then as
	$2^j > \sigma$, for $x \sim N(\mu,\sigma^2)$, by another application of the
	Gaussian CDF,
	\[
		\P{x \in [c2^j, (c+1)2^j)} > \P{x \in [\mu, \mu+\sigma)}
		\geq 0.34.
	\]
	Thus by the same method as above, using the assumption $k > 5000\ln(5/\beta)$,
	with probability at least
	$1-\beta/5$, $\hat H_1^j(c \mod 4) \geq 0.33k - \psi$. By similar logic,
	since
	\[
		\P{\lfloor x/2^j \rfloor \equiv c+2 \mod 4}
		< \max_{\lambda \in [0,2^j]} \P{x \not \in [\mu-2^j - \lambda, \mu+2\cdot 2^j -\lambda]}
		< \P{x \not \in [\mu-\sigma, \mu+2\sigma)} \leq 0.19
	\]
	with probability at least $1-\beta/5$,
	$\hat H_1^j(c+2 \mod 4) \leq 0.2k + \psi$.
	Next, consider $\hat H_1^j(c-1 \mod 4)$. If $\mu \geq (c+0.75)2^j$, then 
	\[
		\P{x \in [(c-1)2^j, c2^j)}
		\leq \P{x \not \in [\mu - 3\sigma/4, \mu + 9\sigma/4]} \leq 0.24
	\]
	so with probability at least $1-\beta/5$
	\[
		\hat H_1^j(c-1 \mod 4) \leq 0.25k + \psi
		< 0.33k - \psi \leq \hat H_1^j(c \mod 4)
	\]
	where the middle inequality uses
	$k > 625\left(\tfrac{\eps+4}{\eps\sqrt{2}}\right)^2\ln(4L/\beta)$. Thus 
	$c \equiv M_1(j)$ or $M_2(j) \mod 4$; the $\mu \leq (c+0.25)2^j)$ case is 
	symmetric. If instead $\mu \in ((c+0.25)2^j, (c+0.75)2^j)$ then by similar 
	logic with probability at least $1-\beta/5$
	\[
		\hat H_1^j(c \mod 4) \geq 0.36k - \psi.
	\]
	so by $\psi < 0.08k$ (implied by
	$k > 40\left(\tfrac{\eps+4}{\eps\sqrt{2}}\right)^2\ln(8L/\beta)$)
	$c \equiv M_1(j)$ or $M_2(j) \mod 4$. It follows that with probability at
	least $1-3\beta/5$ in all cases $c \equiv M_1(j)$ or $M_2(j) \mod 4$.
	Moreover, by a similar application of the Gaussian CDF, one of $c-1 \mod 4$
	and $c+1 \mod 4$ lies in $\{M_1(j),M_2(j)\}$ as well. 
	
	Recalling that $c^*$ is the maximum integer such that $c^*2^j \in I_j$ and 
	$c^* \equiv M_1(j)$ or $M_2(j) \mod 4$, $c^*-1 \mod 4 \in \{M_1(j), M_2(j)\}$
	as well. Assume $|c^*2^j - \mu| > 2\sigma$. By above,
	$\mu \in [c^*2^j, (c^*+1)2^j)$ or $[(c^*-1)2^j, (c^*2^j))$. In the first case,
	\[
		\P{\lfloor x/2^j \rfloor \equiv c^* - 1 \mod 4}
		\leq \P{x \not \in [\mu - 2\sigma, \mu + 2\sigma]} \leq 0.05
	\]
	so with probability at least $1-\beta/5$,
	$\hat H_1^j(c^*-1) \leq 0.06k + \psi$, a contradiction of
	$c^*-1 \mod 4 \in \{M_1(j), M_2(j)\}$. In the second case,
	\[
		\P{\lfloor x/2^j \rfloor \equiv c^* \mod 4}
		\leq \P{x \not \in [\mu - 2\sigma, \mu + 2\sigma]} \leq 0.05
	\]
	and with probability at least $1-\beta/5$,
	$\hat H_1^j(c^*) \leq 0.06k + \psi$, contradicting
	$c^* \mod 4 \in \{M_1(j), M_2(j)\}$. Thus $|c^*2^j - \mu| \leq 2\sigma$.
	
	We put these facts together in \hatmuone~as follows: let $j_1$ be the maximum
	element of $\L$ such that $\hat H_1^j(M_1(j)) < 0.52k - \psi$. If
	$2^{j_1} > \sigma$, then by Fact 2 setting $\hat \mu_1 = c^*2^j$ implies
	$|\hat \mu_1 - \mu| \leq 2\sigma$. If instead $2^{j_1} \leq \sigma$, then
	any setting of $\hat \mu_1 \in I_j$ (including $\hat \mu_1 = c^*2^j$)
	guarantees $|\hat \mu_1 - \mu| \leq 2^{j_1+1} \leq 2\sigma$. Thus in all cases,
	with probability at least $1-\beta$, $|\hat \mu_1 - \mu| \leq 2\sigma$.
\end{proof}

The results above give the analyst an (initial) estimate $\hat \mu_1$ such that $|\hat \mu_1 - \mu| \leq 2\sigma$. This concludes our analysis of round one of \algok . Now, the analyst passes this estimate $\hat \mu_1$ to users $i \in U_2$, and each user uses $\hat \mu_1$ to center their value $x_i$ and randomized respond on the resulting $(x_i - \hat \mu_1)/\sigma$ in \rrtwok . 
The analyst then aggregates these results using \aggtwok . We now prove that this centering process results in a more accurate final estimate $\hat \mu_2$ of $\mu$.

\begin{lemma}
\label{lem:erf_k}
	Conditioned on the success of the previous lemmas, with probability at
	least $1-\beta$ \algok~outputs $\hat \mu_2$ such that
	\[
		|\hat \mu_2 - \mu|
		= O\left(\frac{\sigma}{\eps}\sqrt{\frac{\log(1/\beta)}{n}}\right).
	\]
\end{lemma}
\begin{proof}
	The proof is broadly similar to that of Theorem B.1 in~\citet{BGMNW16}, with
	some modifications for privacy. First, by Lemma~\ref{lem:hatmuonek}
	$\mu - \hat \mu_1 \in [-2\sigma,2\sigma]$.  Letting
	$\bar \mu = (\mu - \hat \mu_1)/\sigma$ we get that
	$x_i' \sim N(\bar \mu, 1)$. Next, since $\E{y_i} = 2\P{x_i' \geq 0} -1$,
	and in general
	\[
		\Phi_{\mu, \sigma^2}(x)
		= \frac{1}{2}\left(1 + \erf\left(\frac{x - \mu}{\sigma\sqrt{2}}\right)\right)
	\]
	where $\Phi_{\mu, \sigma^2}$ is the CDF of $N(\mu, \sigma^2)$, by 
	$\Phi_{\bar \mu, 1}(0) = \P{x_i' \geq 0}$ we get
	$\E{y_i} = \erf(\bar \mu/\sqrt{2})$ (note that we are analyzing the
	unprivatized values $y_i$ to start; later, we will use this analysis to prove
	the analogous result for the privatized values $\tilde y_i$).
	
	A Chernoff bound on $[-1,1]$-bounded random variables then shows that, with
	probability at least $1-\beta/2$, for $y = \frac{2}{n} \sum_{i \in U_2} y_i$ we
	have
	\[
		|y - \erf(\bar \mu/\sqrt{2})| \leq 2\sqrt{\ln(4/\beta)/n}
	\]
	and by $\E{y} = \erf(\bar \mu/\sqrt{2})$ we get
	$|y - \E{y}| \leq 2\sqrt{\ln(4/\beta)/n}$ as well.
	
	Since $\mu - \hat \mu_1 \in [-2\sigma,2\sigma]$,
	$|\erf(\bar \mu/\sqrt{2})| \leq \erf(\sqrt{2})$. Thus
	$|\E{y}| \leq \erf(\sqrt{2})$, so by
	$|y - \E{y}| \leq 2\sqrt{\ln(4/\beta)/n}$ we get
	\[
		|y| \leq \erf(\sqrt{2}) + 2\sqrt{\ln(4/\beta)/n}.
	\]
	Using $n > 20000\ln(4/\beta)$ we get $2\sqrt{\ln(4/\beta)/n} < 0.01$ and
	$\erf(\sqrt{2}) < 0.96$, so $|y| \leq 0.97$ and thus $|y| < \erf(1.6)$.
	Let $M$ be an upper bound on the Lipschitz constant for $\erf^{-1}$ in
	$[-0.97,0.97]$, 
	\begin{align*}
		M =&\; \max_{x \in [-0.97,0.97]} \frac{d \erf^{-1}(x)}{dx} \\
		=&\; \max_{x \in [-0.97,0.97]} \frac{\sqrt{\pi}}{2}
		\exp([\erf^{-1}(x)]^2) \\
		\leq&\; \frac{\sqrt{\pi}}{2}\exp([\erf^{-1}(0.97)]^2) < 10.
	\end{align*}
	Then for any $x, y \in [-0.97,0.97]$ we have $|\erf^{-1}(x) - \erf^{-1}(y)| 
	\leq M|x-y|$, so setting $T = \sqrt{2}\erf^{-1}(y)$,
	\begin{align*}
		|T - \bar \mu| = |\sqrt{2}(\erf^{-1}(y) - \erf^{-1}(\E{y})| 
		\leq&\; 10\sqrt{2}|y - \E{y}| \\
		\leq&\; 20\sqrt{2\ln(4/\beta)/n}
	\end{align*}
	using the bound on $|y - \E{y}|$ from above.
	
	It remains to analyze the privatized values $\{\tilde y_i\}$ and bound
	$|T - \hat T|$, recalling that we set 
	\[
		\hat T =
		\sqrt{2} \cdot \erf^{-1}\left(\frac{2(-\hat H_2(-1) + \hat H_2(1))}{n}\right)
	\]
	in \aggonek. By a Chernoff bound analogous to that of
	Lemma~\ref{lem:hhk}, with probability at least $1-\beta/2$
	\[
		|T - \hat T| \leq \sqrt{2}\left|\erf^{-1}(|y|) - \erf^{-1}\left(|y|
		+ \left[\frac{\eps+2}{\eps}\right]\sqrt{\frac{2\ln(4/\beta)}{n}}\right)\right|.
	\]
	Using $n > 20000\left(\tfrac{\eps+2}{\eps}\right)^2\ln(4/\beta)$ (which implies
	$\left[\frac{\eps+2}{\eps}\right]\sqrt{\frac{2\ln(4/\beta)}{n}} \leq 0.01$) and the
	same derivative trick as above on $[-0.98,0.98]$, we get
	\[
		|T - \hat T| \leq
		14\left[\frac{\eps+2}{\eps}\right]\sqrt{\frac{2\ln(4/\beta)}{n}}.
	\]
	Therefore by the triangle inequality
	\[
		|\hat T - \bar \mu| \leq
		\left(20 + 14\left[\frac{\eps+2}{\eps}\right]\right)
		\sqrt{\frac{2\ln(4/\beta)}{n}}
	\]
	and by $\sigma \bar \mu = \mu - \hat \mu_1$ we get
	\[
		|\sigma\hat T - \sigma\bar \mu|
		= |(\sigma \hat T + \hat \mu_1) - \mu|
		\leq \sigma\left(20 + 14\left[\frac{\eps+2}{\eps}\right]\right)
		\sqrt{\frac{2\ln(4/\beta)}{n}}.
	\]
	Thus by taking $\hat \mu_2 = \sigma \hat T + \hat \mu_1$, we get
	\[
		|\hat \mu_2 - \mu|
		= O\left(\frac{\sigma}{\eps}\sqrt{\frac{\log(1/\beta)}{n}}\right).
	\]
\end{proof}

\section{Proofs from Section~\ref{subsec:kv_one}}
\label{subsec:kv_one_supp}
We start with full pseudocode for \algokone .

\begin{algorithm}[H]
	\caption{\algokone}
	\begin{algorithmic}[1]
		\REQUIRE $\eps, k_1, k_2, \L, n, R, S, \sigma, U_1, U_2$
		\FOR{$j \in \L$}
			\FOR{user $i \in U_1^j$}
				\STATE User $i$ outputs $\tilde y_i \gets \rrone(\eps,i,j)$
			\ENDFOR
		\ENDFOR
		\FOR{$j \in R$}
			\FOR{user $i \in U_2^j$}
				\STATE User $i$ outputs $\tilde y_i \gets \rronekone(\eps, i, S(j))$
			\ENDFOR
		\ENDFOR \COMMENT{End of round 1}
		\STATE Analyst computes $\hat H_1 \gets	\aggonek(\eps, k_1, \L, U_1)$
		\STATE Analyst computes $\hat \mu_1 \gets \hatmuone(\beta, \eps, \hat H_1, k_1, \L,)$
		\STATE Analyst computes $j^* \gets \arg \min_{j \in R}
		\min_{s \in S(j)}|s - \hat \mu_1|$ 
		\STATE Analyst computes	$\hat H_2 \gets \aggtwok(\eps, k_2, U_2^{j*})$
		\STATE Analyst computes
		$\hat T \gets
		\sqrt{2} \cdot \erf^{-1}\left(\frac{-\hat H_2(-1) + \hat H_2(1)}{k_2}\right)$
		\STATE Analyst outputs $\hat \mu_2 \gets
		\sigma \hat T + \arg\min_{s \in S(j^*)}|s - \hat \mu_1|$ 
		\ENSURE Analyst estimate $\hat \mu_2$ of $\mu$
	\end{algorithmic}
\end{algorithm}

\algokone 's privacy guarantee follows from the same analysis of randomized response as in \algok, so we state the guarantee but omit its proof.

\begin{theorem}
\label{thm:algokone_privacy}
	\algokone~satisfies $(\eps,0)$-local differentially privacy for
	$x_1, \ldots, x_n$.
\end{theorem}

We define $k$ (here denoted $k_1$), $\L, U_1$, and $U_2$ as in \algok. As \algokone 's treatment of users in $U_1$ is identical to that of \algok, we skip its analysis, instead recalling its final guarantee:

\begin{lemma}
\label{lem:hatmuonek_two}
	With probability at least $1-\beta$, $|\hat \mu_1 - \mu| \leq 2\sigma$.
\end{lemma}

This brings us to $U_2$, and we define new parameters as follows. For neatness, let $\rho = \lceil 2\sqrt{\ln(4n)}\rceil \geq \lceil \sqrt{2\ln(2\sqrt{n})} + 2.1\rceil$ for $n \geq 32$. We set $R = \{0.2\sigma, 0.4\sigma, \ldots, \rho\sigma\}$ and split $U_2$ into $|R| = 5\rho$ groups indexed by $j \in R$, each of size $k_2 \geq \lfloor n/2|R| \rfloor \geq \lfloor \tfrac{n}{20\sqrt{\ln(4n)}} \rfloor = \Omega(n/\sqrt{\log(n)})$, where the last inequality uses $n \geq 25$. Finally, for each $j \in R$ we define $S(j) = \{j + b\rho\sigma \mid b \in \mathbb{Z}\}$.

With this setup, for each $j \in R$ each user $i \in U_2^j$ uses \rronekone~to execute a group-specific version of \rrtwok: rather than centering by $\hat \mu_1$ as in \rrtwok, user $i$ now centers by the nearest point in $S(j)$ (breaking ties arbitrarily).

\begin{algorithm}[H]
	\caption{\rronekone}
	\begin{algorithmic}[1]
		\REQUIRE $\eps, i, S(j)$
		\STATE User $i$ computes $z_i \gets \arg \min_{z_i \in S(j)} |z_i - x_i|$
		\STATE User $i$ computes $y_i \gets \sgn((x_i-z_i)/\sigma)$
		\STATE User $i$ computes $c \sim_U [0,1]$
		\IF{$c \leq \tfrac{e^\eps}{e^\eps+1}$}
			\STATE User $i$ publishes $\tilde y_i \gets y_i$
		\ELSE
			\STATE User $i$ publishes $\tilde y_i \gets -y_i$
		\ENDIF
		\ENSURE Private centered user estimate $\tilde y_i$
	\end{algorithmic}
\end{algorithm}

To analyze \rronekone, we first prove that users in each group draw points concentrated around $\mu$.

\begin{lemma}
\label{lem:rronekone_cluster}
	With probability at least $1-\beta$, for all $j \in R$, group $U_2^j$ contains $\leq 2\sqrt{k_2}$ users $i$ such that	$|x_i - \mu| > \sigma\sqrt{\ln(4n)}$.
\end{lemma}
\begin{proof}
	First, by a Gaussian tail bound, for each user $i$,
	$\P{|x_i - \mu| \geq \sigma\sqrt{\ln(4n)}} \leq 1/\sqrt{n}$. Let
	$U_C^j$ denote the users in group $U_2^j$ such that
	$|x_i - \mu| > \sigma\sqrt{\ln(4n)}$. Then by a binomial Chernoff
	bound 
	\[
		\P{|U^c| > \frac{k_2}{\sqrt{n}} + \sqrt{\frac{3k_2\ln(|R|/\beta)}{\sqrt{n}}}}
		\leq \beta/|R|
	\]
	so using $n \geq 9\ln(|R|/\beta)^2$ and union bounding over
	$|R| = \Omega(\sqrt{\log(n)})$ groups, the claim follows.
\end{proof}

In particular, this implies that for $j^* = \arg \min_{j^* \in R} \min_{s \in S(j^*)}|s - \hat \mu_1|$ (i.e., the group with element of $S(j^*)$ closest to $\hat \mu_1$), most users draw points in $[\mu - \sigma\sqrt{\ln(4n)}, \mu + \sigma\sqrt{\ln(4n)}]$. Let $s^* = \min_{s \in S(j^*)}|s - \hat \mu_1|$. Our final accuracy result will rely on two facts. First, most users in $U_2^{j^*}$ center using $s^*$. Second, the randomized responses of users who center with $s^*$ are ``almost as good'' as if they were centered by $\mu$. 

\begin{lemma}
\label{lem:kv_one_final}
	Conditioned on the success of the previous lemmas, with probability at least
	$1-\beta$, \algokone~outputs $\hat \mu_2$ such that
	\[
		|\hat \mu_2 - \mu| =
		O\left(\frac{\sigma}{\eps}\sqrt{\frac{\log(1/\beta)\sqrt{\log(n)}}{n}}\right).
	\]
\end{lemma}
\begin{proof}
	Because adjacent points in $R$ are $0.2\sigma$ apart,
	$|s^* - \hat \mu_1| \leq 0.1\sigma$. Lemma~\ref{lem:hatmuonek_two} and the 
	triangle inequality then imply that $|s^* - \mu| \leq 2.1\sigma$. This enables
	us to mimic the proof of Lemma~\ref{lem:erf_k}, replacing 
	$\mu - \hat \mu_1 \in [-2\sigma, 2\sigma]$ with
	$\mu - s^* \in [-2.1\sigma, 2.1\sigma]$. 
	
	We can decompose users in $U_2^{j^*}$ into those with points within
	$\sigma\rho$ of $s^*$ and those with more distant points.
	Denote the first set of users by $V$ and the second set by $V^c$, and recall
	that the Gaussian CDF is
	\[
		\Phi_{\mu, \sigma^2}(x)
		= \frac{1}{2}\left(1 + \erf\left(\frac{x - \mu}{\sigma\sqrt{2}}\right)\right).
	\]
	Then, letting $\ind$ denote the indicator function,
	\begin{align*}
		\E{y_i \cdot \ind(i \in V)} =&\;
		\P{y_i = 1, i \in V} - \P{y_i = -1, i \in V} \\
		=&\; \Phi_{\mu,\sigma^2}(s^*+\sigma\rho)
		+ \Phi_{\mu,\sigma^2}(s^* - \sigma\rho)
		-2\Phi_{\mu,\sigma^2}(s^*) \\
		=&\; \frac{1}{2}\left[\erf\left(\frac{s^* + \sigma\rho - \mu}{\sigma\sqrt{2}}\right)
		+ \erf\left(\frac{s^* - \sigma\rho - \mu}{\sigma\sqrt{2}}\right)\right]
		- \erf\left(\frac{s^* - \mu}{\sigma\sqrt{2}}\right) \\
		=&\; \frac{1}{2}\left[\erf\left(\frac{\sigma\rho + s^* - \mu}{\sigma\sqrt{2}}\right)
		- \erf\left(\frac{\sigma \rho - (s^* - \mu)}{\sigma\sqrt{2}}\right)\right]
		- \erf\left(\frac{s^* - \mu}{\sigma\sqrt{2}}\right).
	\end{align*}
	where the last step uses the fact that $\erf$ is an odd function.
	Since $\erf(x) = \tfrac{2}{\sqrt{\pi}} \int_0^x e^{-t^2} dt$ and
	$|s^* - \mu| \leq 2.1\sigma$,
	\begin{align*}
		\frac{1}{2}\left[\erf\left(\frac{\sigma\rho + s^* - \mu}{\sigma\sqrt{2}}\right)
		- \erf\left(\frac{\sigma \rho - (s^* - \mu)}{\sigma\sqrt{2}}\right)\right]
		\leq&\; \frac{1}{\sqrt{\pi}}
		\int_{(\sigma \rho - 2.1\sigma)/\sigma\sqrt{2}}^{(\sigma \rho + 2.1\sigma)/\sigma\sqrt{2}}
		e^{-t^2}dt \\
		<&\; 3e^{-[(\rho-2.1)/\sqrt{2}]^2} \\
		\leq&\; 3e^{-\ln(4n)/2}
	\end{align*}
	where the second inequality relies on $e^{-x}$ being monotone decreasing and
	the last step uses $n > 20$, which implies
	$\rho - 2.1 \geq \sqrt{\ln(4n)}$. Then using $n \geq 3k_2$ we get
	$3e^{-\ln(4n)/2} \leq \tfrac{1}{\sqrt{k_2}}$, so
	\begin{equation}
		\left|\E{y_i \cdot \ind(i \in V)} - 
		\erf\left(\frac{\mu - s^*}{\sigma\sqrt{2}}\right)\right|
		\leq \frac{1}{\sqrt{k_2}}.
	\end{equation}
	Next, as $|s^* - \mu| \leq 2.1\sigma$, users having points within
	$\sigma\sqrt{2\ln(2\sqrt{n})}$ of $\mu$ have points within $\sigma\rho$ of
	$s^*$. The Gaussian tail bound from Lemma~\ref{lem:rronekone_cluster} then
	implies $\P{x \in V^c} \leq 1/\sqrt{n}$.
	$\E{y_i} = \E{y_i \cdot \ind(i \in V)} + \E{y_i \cdot \ind(i \in V^c)}$, and
	by the above bound on $\P{x \in V^c}$ and $|y_i| \leq 1$ we get
	$|\E{y_i \cdot \ind(i \in V^c)}| \leq 1/\sqrt{n}$. Thus
	\begin{equation}
		\left|\E{y_i \cdot \ind(i \in V)} - \E{y_i}\right| \leq \frac{1}{\sqrt{n}}
		< \frac{1}{\sqrt{k_2}}.
	\end{equation}
	A Chernoff bound on $\{-1,1\}$-valued random variables then tells us that, for
	$y = \tfrac{1}{k_2}\sum_{i \in U_2^{j^*}} y_i$, with probability at least
	$1-\beta/2$ we have
	\begin{equation}
		\left|y- \E{y_i}\right| \leq \sqrt{\frac{2\ln(4/\beta)}{k_2}}.
	\end{equation}
	Combining the three numbered equations above with the triangle inequality
	yields
	\[
		\left|y - \erf\left(\frac{\mu - s^*}{\sigma\sqrt{2}}\right)\right|
		< \frac{2 + \sqrt{2\ln(4/\beta)}}{\sqrt{k_2}}.
	\]
	Setting	$\bar \mu = (\mu - s^*)/\sigma$ and using
	$k_2 \geq (100[2+\sqrt{2\ln(4/\beta)})])^2$, this rearranges into
	$|y| \leq \erf(\bar \mu/\sqrt{2}) + 0.01$. Since $\bar \mu \in [-2.1,2.1]$, we
	get
	\[
		|y| < \erf(2.1/\sqrt{2}) + 0.01 < 0.98 < \erf(1.7).
	\]
	
	Let $M$ be an upper bound on the Lipschitz constant for $\erf^{-1}$ in
	$[-0.98,0.98]$,
	\begin{align*}
		M =&\; \max_{x \in [-0.98,0.98]} \frac{d \erf^{-1}(x)}{dx} \\
		=&\; \max_{x \in [-0.98,0.98]} \frac{\sqrt{\pi}}{2}
		\exp([\erf^{-1}(x)]^2) \\
		\leq&\; \frac{\sqrt{\pi}}{2}\exp([\erf^{-1}(0.98)]^2) < 14.
	\end{align*}
	Then for any $x, y \in [-0.98,0.98]$ we have $|\erf^{-1}(x) - \erf^{-1}(y)| 
	\leq M|x-y|$, so for $T = \sqrt{2}\erf^{-1}(y)$,
	\begin{align*}
		|T - \bar \mu| = |\sqrt{2}(\erf^{-1}(y) - \erf^{-1}(\erf(\bar \mu/\sqrt{2}))| 
		\leq&\; 14\sqrt{2}|y - \erf(\bar \mu/\sqrt{2})| \\
		<&\; 28\left(\frac{\sqrt{2} + \sqrt{\ln(4/\beta)}}{k_2}\right).
	\end{align*}
	
	It remains to bound $|T - \hat T|$, where $T$ is the (unknown) aggregation of
	unprivatized $\{y_i\}$ while $\hat T$ is the (known) aggregation of privatized
	$\{\tilde y_i\}$. By a Chernoff bound analogous to that of
	Lemma~\ref{lem:hhk}, with probability at least $1-\beta/2$
	\[
		|T - \hat T| \leq \sqrt{2}\left|\erf^{-1}(|y|) - \erf^{-1}\left(|y|
		+ \left[\frac{\eps+2}{\eps}\right]\sqrt{\frac{2\ln(4/\beta)}{k_2}}\right)\right|.
	\]
	Using $k_2 > 20000\left(\tfrac{\eps+2}{\eps}\right)^2\ln(4/\beta)$ (which
	implies
	$\left[\frac{\eps+2}{\eps}\right]\sqrt{2\frac{\ln(4/\beta)}{k_2}} \leq 0.01$)
	and the same derivative trick as above on $[-0.99,0.99]$, we get
	\[
		|T - \hat T| \leq
		25\left[\frac{\eps+2}{\eps}\right]\sqrt{\frac{2\ln(4/\beta)}{k_2}}.
	\]
	Therefore by the triangle inequality
	\[
		|\hat T - \bar \mu| \leq
		28\left(\frac{\sqrt{2} + \sqrt{\ln(4/\beta)}}{k_2}\right)
		+ 25\left[\frac{\eps+2}{\eps}\right]\sqrt{\frac{2\ln(4/\beta)}{k_2}}
		= O\left(\frac{1}{\eps}\sqrt{\frac{\log(1/\beta)}{k_2}}\right)
	\]
	and by $\sigma \bar \mu = \mu - s^*$ we get
	\[
		|\sigma\hat T - \sigma\bar \mu|
		= |(\sigma \hat T + s^*) - \mu|
		= O\left(\frac{\sigma}{\eps}\sqrt{\frac{\log(1/\beta)}{k_2}}\right).
	\]
	Thus by taking $\hat \mu_2 = \sigma \hat T + s^*$ and substituting in
	$k_2 = \Omega(n/\sqrt{\log(n)})$ we get
	\[
		|\hat \mu_2 - \mu|
		= O\left(\frac{\sigma}{\eps}\sqrt{\frac{\log(1/\beta)\sqrt{\log(n)}}{n}}\right).
	\]
\end{proof}

\section{Proofs from Section~\ref{subsec:uv_two}}
\label{subsec:uv_two_supp}
We begin our analysis with a privacy guarantee.

\begin{theorem}
	\algo~satisfies $(\eps,0)$-local differentially privacy for $x_1, \ldots, x_n$.
\end{theorem}
\begin{proof}
	As we already proved that \rrone~is private in Section~\ref{subsec:kv_two_supp}, we
	are left with \rrtwo. To prove that \rrtwo~is $(\eps,0)$-locally
	differentially private as well, we can use a standard Laplace noise privacy
	guarantee (see e.g. Theorem 3.6 from Dwork and Roth~\cite{DR14}): given function $f$ with
	1-sensitivity $\Delta f$, computing $f(x) + \Lap{\Delta f/\eps}$ satisfies
	$(\eps,0)$-differential privacy.
\end{proof}

First, for each $j \in \L_1$ and $i \in U_1^j$, user $i$ employs \rrone~(see Section~\ref{subsec:kv_two_supp}) to publish a privatized version of $\lfloor x/2^j \rfloor \mod 4$. The analyst then constructs two slightly different aggregations of this data. To estimate $\sigma$, the analyst aggregates responses into $\hat H_1$ via \aggone , which is similar to \aggonek~up to the choice of bins in the constructed histogram $\hat H_1$. Specifically, bins in $\hat H_1$ are grouped: points with value 0 count toward both bin $(0,1)$ and bin $(3,0)$, points with value 1 count toward both bin $(0,1)$ and $(1,2)$, and so on.

\begin{algorithm}[h]
	\caption{\aggone}
	\begin{algorithmic}[1]
		\REQUIRE $\eps, k, \L, U$
		\FOR{$j \in \L$}
			\FOR{$a = 0, 1, 2, 3$}
				\STATE Analyst computes
				$C^j(a) \gets |\{i \mid i \in U^j, \tilde y_i = a\}|$
				\STATE Analyst computes
				$\hat H^j(a) \gets \frac{e^\eps+3}{e^\eps-1} \cdot
				\left(C^j(a) - \frac{k}{e^\eps+3}\right)$
			\ENDFOR
			\FOR{$a = 0,1,2,3$}
				\STATE Analyst computes $\hat H_1^j(a) \gets
				\hat H^j(a) + \hat H^j(a + 1 \mod 4)$
			\ENDFOR
		\ENDFOR		
		\STATE Analyst outputs $\hat H_1$
		\ENSURE Analyst aggregation $\hat H_1$ of private user estimates
		\end{algorithmic}
\end{algorithm}

At a high level, when $2^j \gg \sigma$, user responses in group $U_1^j$ appear concentrated in one element of $\{(0,1), (1,2), (2,3), (3,0)\}$. This is because user data comes from $N(\mu,\sigma^2)$, so if $2^j \gg \sigma$ then most user data falls within $3\sigma$ of $\mu$. Consequently, there exists $a \in \{0,1,2,3\}$ such that most users draw points $x$ where $\lfloor x/2^j \rfloor \equiv a$ or $a+1 \mod 4$, and $\hat H_1^j$ is concentrated around bin $(a,a+1 \mod 4)$. Similarly, if $2^j \ll \sigma$ then user responses in group $U_1^j$ appear unconcentrated (for a more precise definition of ``concentrated'', see below).

Examining this transition from concentrated to unconcentrated responses in $\hat H_1^{\lmax}, \hat H_1^{\lmax-1}, \ldots$ yields a rough estimate of when $2^j \gg \sigma$ versus when $2^j \ll \sigma$. By approximating when this change occurs, the analyst recovers an approximation of $\sigma$.
This process is outlined in \hatsigma.

\begin{algorithm}[h]
	\caption{\hatsigma}
	\begin{algorithmic}[1]
		\REQUIRE $\beta, \eps, \hat H_1, k_1, \L_1$
		\STATE Analyst computes $j \gets$ minimum $j$ such that, for all
		$j' \geq j$, $\hat H_1^{j'}$ is concentrated
		\IF{$j = \emptyset$}
			\STATE Analyst outputs
			$\hat \sigma \gets 2^{\lmax}$
		\ELSE
		 	\STATE Analyst outputs $\hat \sigma \gets 2^{j}$
		\ENDIF
		\ENSURE Analyst estimate $\hat \sigma$ of $\sigma$
	\end{algorithmic}
\end{algorithm}

$\hat H_1$ is an estimate of the ``true'' histogram collection, $H^j(a) = |\{y_i \mid i \in U_1^j, y_i \in \{a, a+1 \mod 4\}\}|$ for all $j \in \L_1$. As in Lemma~\ref{lem:hhk}, we can show that $\hat H_1$ and $H_1$ are similar. As the proof is nearly identical, we omit it.

\begin{lemma}
\label{lem:hh_2}
	With probability at least $1-\beta$, for all $j \in \L_1$,
	\[
		||\hat H_1^j - H_1^j||_\infty \leq
		\left(1 + \tfrac{4}{\eps}\right)\sqrt{2k_1\ln(8L_1/\beta)}.
	\]
\end{lemma}

Next, we show how the analyst uses $\hat H_1$ to estimate $\sigma$ in subroutine \hatsigma. Here, for neatness we shorthand
\[
	\tau = \sqrt{2k_1\ln(2L_1/\beta)} +
	\left(1 + \tfrac{4}{\eps}\right) \sqrt{2k_1\ln(8L_1/\beta)}
\]
and use the term \emph{concentrated} to denote any histogram $\hat H_1^j$ such that $\min_{a \in \{0,1,2,3\}} \hat H_1^j(a) \leq 0.03k + \tau$ and the term \emph{unconcentrated} to denote $\hat H_1^j$ where $\min_{a \in \{0,1,2,3\}}\hat H_1^j(a) \geq 0.04k - \tau$. As we show below in Lemma~\ref{lem:sigma}, when $2^j \gg \sigma$, $\hat H_1^j$ is concentrated. Similarly, when $2^j \ll \sigma$, $\hat H_1^j$ is unconcentrated. This transition enables the analyst to estimate $\sigma$.

\begin{lemma}
\label{lem:sigma}
	Conditioned on the success of the preceding lemmas, with probability at least
	$1-\beta$, \hatsigma~outputs $\hat \sigma \in [\sigma, 8\sigma]$.
\end{lemma}
\begin{proof}
	Choose $j \in \L_1$. Below, we reason about two (non-exhaustive)
	possibilities for $j$.
	
	\underline{Case 1}: $2^j \geq 4\sigma$. Then there exists $a \in \{0,1,2,3\}$
	and interval $I$ of length $2^{j+1} \geq 8\sigma$ containing
	$[\mu - 2\sigma, \mu + 2\sigma]$ such that for all $x \in I$,
	$\lfloor x/2^j \rfloor \mod 4 \equiv a$ or $a+1 \mod 4$. By similar 
	application of the Gaussian CDF as in Lemma~\ref{lem:hatmuonek}, with
	probability at least $1-\beta/2L_1$,
	\[
		|\{x_i \mid x_i \in I, i \in U_1^j\}| \geq
		0.97k_1 - \sqrt{2k_1\ln(2L_1/\beta)}.
	\]
	Thus by Lemma~\ref{lem:hh_2}, $\hat H^j_1(a) \geq 0.97k_1 - \tau$. It follows
	that $\hat H^j_1(a+2) \leq 0.03k_1 + \tau$. $2^j \geq 4\sigma$ thus implies
	that histogram $\hat H^j_1$ is concentrated.
	
	\underline{Case 2}: $2^j \in [\sigma/2,\sigma]$. Choose
	$a \in \{0,1,2,3\}$. Since $2^j \in [\sigma/2,\sigma]$ there exist at most 
	three subintervals $I_1, I_2, I_3 \subset [\mu-2\sigma, \mu+2\sigma]$ such 
	that for all $x \in I = I_1 \cup I_2 \cup I_3$,
	$\lfloor x/2^j \rfloor \equiv a \mod 4$, and $|I| \geq \sigma$. Let 
	$x \sim N(\mu,\sigma^2)$. Then by a similar application of the Gaussian CDF as
	in Lemma~\ref{lem:hatmuonek}, since
	\[
		\P{x \in I} \geq \P{x \in [\mu-2\sigma,\mu-\sigma)} \geq 0.13
	\]
	with probability $1-\beta/8L_1$ at least $0.13k - \sqrt{2k_1\ln(8L_1/\beta)}$
	users from $U_1^j$ have points in $I$. Since this held for arbitrary $a$, a
	union bound over all four possibilities of $a$ combined with
	Lemma~\ref{lem:hh_2} implies that, with probability at least $1-\beta/2L_1$,
	\[
		\min_{a \in \{0,1,2,3\}}\hat H_1^j(a) \geq 0.13k_1 - \tau.
	\]
	$2^j \leq \sigma \leq 2^{j+1}$ thus implies that histogram $\hat H_1^j$ is 
	uniform. 
	
	Union bounding both results over $j \in \L_1$, with
	$k_1 > 800\left(2+\tfrac{4}{\eps}\right)^2\ln(8L_1/\beta)$, with probability
	$1-\beta$ we have $0.13k - \tau > 0.03k + \tau$ for each $j \in \L_1$.
	Therefore for all $j \in \L_1$ if $2^j \geq 4\sigma$ then $\hat H_1^j$ will be
	concentrated while if $2^{j+1} \geq \sigma \geq 2^j$ then $\hat H_1^j$ will be
	unconcentrated.
	
	Let $j$ be the smallest $j \in \L_1$ such that $\hat H_1^j$ is concentrated 
	and for all $j' > j$, $\hat H_1^{j'}$ is concentrated as well. If no such
	$j$ exists, then we know $2^{\lmax} \geq \sigma \geq 2^{\lmax-2}$, take
	$\hat \sigma = 2^{\lmax}$, and we get $\hat \sigma \in [\sigma, 4\sigma]$. If
	not, then by Case 1 above we know $2^j \leq 8\sigma$, and by Case 2 we know
	$2^j \geq \sigma$. Thus taking $\hat \sigma = 2^j$, we get
	$\hat \sigma \in [\sigma, 8\sigma]$.
\end{proof}

Next, the analyst uses randomized responses from $U_1$ to compute an initial estimate $\hat \mu_1$ of $\mu$. As the process \hatmuone~is identical to that used in \algok~up to a different subgroup range $\L_1$, we skip its description and only recall its guarantee: 

\begin{lemma}
\label{lem:hatmu}
	Conditioned on the success of the preceding lemmas, with probability at
	least $1-\beta$, $|\hat \mu_1 - \mu| \leq 2\sigma$.
\end{lemma}

From the results above, the analyst obtains an estimate $\hat \sigma$ such that $\hat \sigma \in [\sigma,8\sigma]$ and an estimate $\hat \mu_1$ such that $|\hat \mu_1 - \mu| \leq 2\sigma$. The analyst now uses these to compute interval $I = [\hat \mu_1 - \hat \sigma(2 + \sqrt{\ln(4n)}), \hat \mu_1 + \hat \sigma(2 + \sqrt{\ln(4n)})]$, where $I$ is intentionally constructed to (with high probability) contain the points of 
$\Omega(n)$ users. The analyst then passes $I$ to users in $U_2$. Users in $U_2$ respond with noisy responses via independent calls to \rrtwo. In \rrtwo, each user clips their sample $x_i$ to the interval $I$ and reports a private version $\tilde y_i$ using Laplace noise scaled to $|I|$.

\begin{algorithm}[h]
	\caption{\rrtwo}
	\begin{algorithmic}[1]
		\REQUIRE $\eps, i, I$
		\STATE User $i$ computes $x_i' \gets \arg \min_{x \in I}|x-x_i|$
		\STATE User $i$ outputs $\tilde y_i \gets x_i' + \Lap{|I|/\eps}$
		\ENSURE Private version of user's point clipped to $I$
	\end{algorithmic}
\end{algorithm}

The average of these $\tilde y_i$ then approximates $\mu$. We formalize this in the following lemma, which proves our main result.

\begin{lemma}
\label{lem:hatmutwo}
	Conditioned on the success of the previous lemmas, with probability at least
	$1-\beta$, $|\hat \mu_2 - \mu|
	= O\left(\tfrac{\sigma}{\eps}\sqrt{\tfrac{\log(1/\beta)\log(n)}{n}}\right)$.
\end{lemma}
\begin{proof}
	There are two sources of error in the analyst's estimate
	$\hat \mu_2 = \tfrac{2}{n}\sum_i \tilde y_i$: error from the unnoised $x_i'$s
	and error from noise in $\tilde y_i$s. Specifically, recalling that
	$|U_2| = n/2$, we can decompose	$\hat \mu_2$ as
	\[
		\hat \mu_2 = \frac{2}{n} \sum_i \tilde y_i
		= \frac{2}{n}\sum_i (x_i' + \eta_i)
	\]
	where each $\eta_i \sim _{i.i.d.}\Lap{|I|/\eps}$ and
	$|I| = 2\hat \sigma(2 + \sqrt{\ln(4n)})$.
	
	First, using $n > 4\ln(3/\beta)$ by concentration of independent Laplace 
	random variables (see e.g. Lemma 2.8 in~\citet{CSS11}) with probability at
	least $1-\beta/3$,
	\[
		\left| \frac{2}{n} \sum_i \eta_i \right|
		\leq \frac{4|I|}{\eps}\sqrt{\frac{2\ln(3/\beta)}{n}}
		\leq \frac{8\hat \sigma(2 + \sqrt{\ln(4n)})}{\eps}
		\sqrt{\frac{2\ln(3/\beta)}{n}}
		= O\left(\frac{\hat \sigma}{\eps}\sqrt{\frac{\log(1/\beta)\log(n)}{n}}
		\right).
	\]
	This bounds the contribution of Laplace noise to overall error.
	
	It remains to bound $|\frac{2}{n} \sum_i x_i' - \mu|$. Let $V$ denote the set
	of users with $x_i \in I$ and $V^c$ denote the set of users with
	$x_i \not \in I$. First, by a Gaussian tail bound, for each user $i$,
	$\P{|x_i - \mu| \geq \sigma\sqrt{\ln(4n)}} \leq 1/\sqrt{n}$. Then by
	a Chernoff bound 
	\[
		\P{|V^c| > \left(1 + \sqrt{\frac{6\ln(3/\beta)}{n^{3/2}}}\right)\sqrt{n}}
		\leq \beta/3
	\]
	and using $n \geq (6\ln(2/\beta))^{2/3}$ we get 
	$\sqrt{\frac{6\ln(3/\beta)}{n^{3/2}}} \leq 1$, so with probability at least
	$1-\beta/3$, $|V^c| \leq 2\sqrt{n}$. Thus
	\[
		\frac{2}{n}\sum_{i \in V^c} |x_i' - \mu| \leq \frac{2}{n}(|V^c| \cdot |I|)
		\leq \frac{6 \hat \sigma(2 + \sqrt{\ln(4n)})}{\sqrt{n}}
		= O\left(\frac{\hat \sigma \sqrt{\log(n)}}{\sqrt{n}}\right).
	\]
	This bounds the contribution of error from the (unprivatized) data of users
	in $V^c$. Let $V$ denote the set of users in $U_2$ with points in $I$. We
	bound the error contributed by users in $V$ in a similar way. Users in $V$
	have $x_i' = x_i$, so by a Chernoff bound on (shifted) $[0,|I|]$-bounded
	random variables, with probability at least $1-\beta/3$
	\[
		\frac{2}{n}\sum_{i \in V^c} |x_i' - \mu|
		= \frac{2}{n}\sum_{i \in V^c} |x_i - \mu| \leq
		|I|\sqrt{\frac{2\ln(6/\beta)}{n}} \leq
		\hat \sigma(2 + \sqrt{\ln(4n)})\sqrt{\frac{2\ln(6/\beta)}{n}}
		= O\left(\frac{\hat \sigma \sqrt{\log(1/\beta)\log(n)}}{\sqrt{n}}\right).
	\]
	Putting these three bounds together, we get
	\begin{align*}
		\left|\frac{2}{n} \sum_i \tilde y_i - \mu\right|
		\leq&\; \frac{2}{n} \sum_i |x_i' + \eta_i - \mu| \\
		\leq&\; \frac{2}{n} \sum_i |x_i' - \mu| + \frac{2}{n}\sum_i |\eta_i| \\
		=&\; \frac{2}{n} \sum_{i \in V} |x_i' - \mu| 
		+ \frac{2}{n}\sum_{i \in V^c} |x_i' - \mu|
		+ \frac{2}{n}\sum_i |\eta_i| \\
		=&\; O\left(\frac{\sigma}{\eps}\sqrt{\frac{\log(n)\log(1/\beta)}{n}}\right)
	\end{align*}
	where the last step uses $\hat \sigma \in [\sigma,8\sigma]$ from 
	Lemma~\ref{lem:sigma}.
\end{proof}

\section{Proofs from Section~\ref{subsec:uv_one}}
\label{subsec:uv_one_supp}
We start with pseudocode for \algoone .

\begin{algorithm}[H]
	\caption{\algoone}
	\begin{algorithmic}[1]
		\REQUIRE $\eps, k_1, k_2, \L_1, n, R, \sigma, U_1, U_2,$
		\STATE Analyst computes $\rho \gets \lceil \sqrt{2\ln(2\sqrt{n})} + 6\rceil$
		\FOR{$j \in \L_1$}
			\FOR{user $i \in U_1^j$}
				\STATE User $i$ outputs $\tilde y_i \gets \rrone(\eps,i,j)$
			\ENDFOR
		\ENDFOR
		\FOR{$j_1 \in \L_1$}
			\FOR{$j_2 \in R_{j_1}$}
				\FOR{user $i \in U_2^{j_1,j_2}$}
					\STATE User $i$ outputs $\tilde y_i \gets
					\rrthree(\eps, i, j_1, j_2, \rho, S)$
				\ENDFOR
			\ENDFOR
		\ENDFOR \COMMENT{End of round 1}
		\STATE Analyst computes	$\hat H_1 \gets \aggone(\eps, k_1, \L_1, U_1)$
		\STATE Analyst computes $\hat \sigma \gets
		\hatsigma(\beta, \eps, \hat H_1, k_1, \L)$
		\STATE Analyst computes $j_1 \gets \log(\hat \sigma)$
		\STATE Analyst computes $\hat H_2 \gets \aggonek(\eps, k_1, \L_1, U_1)$
		\STATE Analyst computes
		$\hat \mu_1 \gets \hatmuone(\beta, \eps, \hat H_2, k_1, \L_1)$
		\STATE Analyst computes $j_2 \gets
		\arg \min_{j \in R_{j_1}}(\min_{s \in S(j_1,j)} |s - \hat \mu_1|)$	
		\STATE Analyst computes $s^* \gets \min_{s \in S(j_1,j_2)} |s - \hat \mu_1|$
		\STATE Analyst outputs 	$\hat \mu_2 \gets
		s^* + \tfrac{1}{k_2}\sum_{i \in U_2^{j_1,j_2}} \tilde y_i$ 
		\ENSURE Analyst estimate $\hat \mu_2$ of $\mu$
	\end{algorithmic}
\end{algorithm}

\algoone 's privacy guarantee follows from the same analysis of randomized response and Laplace noise as for \algo, so we omit its proof.

\begin{theorem}
\label{thm:algoone_privacy}
	\algoone~satisfies $(\eps,0)$-local differentially privacy for
	$x_1, \ldots, x_n$.
\end{theorem}

We define $k_1, \L_1,$ and $U_1$, as in \algo~and skip the analysis of \algoone 's treatment of users in $U_1$ as it is identical to that of \algo. We recall its collected guarantee:

\begin{lemma}
\label{lem:algoone_restated}
	With probability at least $1-\beta$, $\hat \sigma \in [\sigma, 8\sigma]$ and
	$|\hat \mu_1 - \mu| \leq 2\sigma$.
\end{lemma}

We again define $R$ and $S$ for $U_2$, albeit with a few modifications. First, we let $\rho = \lceil \sqrt{\ln(4n)} + 6\rceil$ for neatness. Then, recalling from Section~\ref{subsec:uv_two_supp} that $\L_1$ ranges over possible values of $\log(\sigma)$, for each $j_a \in \L_1$ we define $R_{j_a} = \{2^{j_a}, 2 \cdot 2^{j_a}, \ldots, \rho \cdot 2^{j_a}\}$. Next, for each $j_a \in \L_1$ and $j_b \in R_{j_a}$, we define  $S(j_a,j_b) = \{j_b + b\rho 2^{j_a} \mid b \in \mathbb{Z}\}$. Finally, we split $U_2$ into $L_1 \cdot \rho$ subgroups $U_2^{j_a,j_b}$ of size $k_2 = \Omega\left(\tfrac{n}{\log\left(\tfrac{\sigma_{\max}}{\sigma_{\min}}+1\right) \sqrt{\log(n)}}\right)$ for each $j_a \in \L_1$ and $j_b \in R_{j_a}$. As in \algokone, we parallelize over these subgroups to simulate the second round of \algo~for different values of $(j_a,j_b)$. 

In each subgroup $U_2^{j_a,j_b}$, each user $i$ computes the nearest element $s_i \in S(j_a,j_b)$ to $x_i$, $s_i = \arg \min_{s \in S(j_a,j_b)} |x_i - s|$ and outputs $x_i - s_i$ plus Laplace noise in \rrthree. The analyst then uses estimates $j_1 = \lceil\log(\hat \sigma)\rceil$ and $\hat \mu_1$ from $U_1$ to compute $j_2 = \arg \min_{j \in R_{j_1}} (\min_{z \in S(j_1,j)} |z - \hat \mu_1|)$. Finally, the analyst aggregates randomized responses from group $U_2^{j_1,\hat \mu_2}$ into an estimate $\hat \mu_2$.

\begin{algorithm}[H]
	\caption{\rrthree}
	\begin{algorithmic}[1]
		\REQUIRE $\eps, i, j_1, j_2, \rho, S$
		\STATE User $i$ computes $s_i \gets \min_{s \in S(j_1,j_2)}|s - x_i|$
		\STATE User $i$ computes $y_i \gets x_i - s_i$
		\STATE User $i$ outputs $\tilde y_i \gets y_i + \Lap{2\rho 2^{j_1}/\eps}$
		\ENSURE Private version of user's point $x_i$
	\end{algorithmic}
\end{algorithm}

As in \algokone, we start with a concentration result for each $U_2^{j_1,j_2}$. Since its proof is similar to that of Lemma~\ref{lem:rronekone_cluster}, we omit it.

\begin{lemma}
\label{lem:rrfour_cluster}
	With probability at least $1-\beta$, for all $j_1 \in \L_1$ and $j_2 \in R_{j_1}$, group $U_2^{j_1,j_2}$ contains $\leq 2\sqrt{k_2}$ users $i$ such that $|x_i - \mu| > \sigma\sqrt{\ln(4n)}$.
\end{lemma}

In combination with the previous lemmas, this enables us to prove our final accuracy result.

\begin{lemma}
	Conditioned on the success of the previous lemmas, with probability at least $1-\beta$, \algoone~outputs $\hat \mu_2$ such that
	\[
		|\hat \mu_2 - \mu| = O\left(\frac{\sigma}{\eps}
		\sqrt{\frac{\log\left(\frac{\sigma_{\max}}{\sigma_{\min}} + 1\right)
		\log(1/\beta)\log^{3/2}(n)}{n}}\right).
	\]
\end{lemma}
\begin{proof}
	By Lemma~\ref{lem:algoone_restated}, $\hat \sigma \in [\sigma, 8\sigma]$ and
	$|\hat \mu_1 - \mu| \leq 2\sigma$. Since $j_1 = \log(\hat \sigma)\in \L_1$
	and \\
	$j_2 = \arg \min_{j \in R_{j_1}}(\min_{s \in S(j_1,j)} |s - \hat \mu_1|)$,
	by the definition of $s^* \in S(j_1,j_2)$,
	$|s^* - \hat \mu_1| \leq 0.5\hat \sigma < 4\sigma$. Thus
	$|s^* - \mu| < 6\sigma$.
	
	Consider group $U_2^{j_1,j_2}$. By Lemma~\ref{lem:rrfour_cluster} at most 
	$2\sqrt{k_2}$ users $i \in U_2^{j_1,j_2}$ have
	$|x_i - \mu| > \sigma\sqrt{\ln(4n)}$. Thus by
	$|s^* - \mu| < 6\sigma$ and the fact that any two points in
	$S(j_1,j_2)$ are at least
	$\hat \sigma \rho \geq \sigma(6 + \sqrt{\ln(4n)})$ far apart,
	we get that at least $k_2 - 2\sqrt{k_2}$ users
	$i \in U_2^{j_1,j_2}$ set $s_i = s^*$ in their run of \rrthree. Denote this
	subset of users by $V$, and denote by $V^c$ the set of users
	$i \in U_2^{j_1,j_2}$ such that $s_i \neq s^*$, and for each user
	$i \in U_2$ let $y_i = x_i - s_i$. 

	Let $f(x) = \tfrac{1}{\sigma\sqrt{2\pi}}\exp(-(x-\mu)^2/2\sigma^2)$, the density
	for $N(\mu,\sigma^2)$. Then
	\begin{equation}
	\label{eq:int_sum}
		\int_{\infty}^{\infty} (x-\mu)f(x)dx
		= \int_{-\infty}^{s^* - \rho\hat \sigma} (x-\mu)f(x)dx
		+ \int_{s^* - \rho\hat \sigma}^{s^* + \rho\hat \sigma} (x-\mu)f(x)dx
		+ \int_{s^* + \rho\hat \sigma}^{\infty} (x-\mu)f(x)dx.
	\end{equation}
	Let $g(x) = -\tfrac{\sigma}{\sqrt{2\pi}}\exp(-(x-\mu)^2/2\sigma^2)$, the
	antiderivative of $(x-\mu)f(x)$. Then
	\begin{align*}
		\left|\int_{-\infty}^{s^* - \rho\hat\sigma} (x-\mu)f(x)dx\right|
		=&\; \left|g(s^* - \rho\hat\sigma) - \lim_{b \to -\infty} g(b)\right| \\
		=&\; \left|\frac{\sigma}{\sqrt{2\pi}} \cdot \exp\left(-\frac{(s^* - \rho\hat\sigma - \mu)^2}{2\sigma^2}\right)\right| \\
		\leq&\; \left|\frac{\sigma}{\sqrt{2\pi}} \cdot \exp\left(-\frac{([6-\rho]\sigma)^2}{2\sigma^2}\right)\right| \\
		\leq&\; \left|\frac{\sigma}{\sqrt{2\pi}} \cdot \exp\left(-\frac{[6-\rho]^2}{2}\right)\right|\\
		<&\; \frac{\sigma}{\sqrt{2\pi}} \cdot \exp(-\ln(2\sqrt{n})) \\
		<&\; \frac{\sigma}{\sqrt{n}}
	\end{align*}
	where the first inequality uses $\hat \sigma \geq \sigma$ and
	$|s^* - \mu| < 6\sigma$. Similar logic implies \\
	$\left|\int_{s^* + \rho\hat\sigma}^{\infty} (x-\mu)f(x)dx\right| \leq \sigma/\sqrt{n}$
	as well. Therefore by Equation~\ref{eq:int_sum} and 
	$\int_{-\infty}^{\infty} (x-\mu)f(x)dx = 0$,
	\[
		\left|\int_{s^* - \rho\hat\sigma}^{s^* + \rho\hat\sigma} (x-\mu)f(x)dx\right|
		\leq 2\sigma/\sqrt{n}
	\]
	so by
	$\E{x_i \cdot \ind(i \in V)} = \int_{s^* - \rho\hat\sigma}^{s^* + \rho\hat\sigma}
	xf(x)dx$, we get
	\[
		\left|\E{x_i \cdot \ind(i \in V)}
		- \mu\int_{s^* - \rho\hat\sigma}^{s^* + \rho\hat\sigma}f(x)dx\right|
		\leq 2\sigma/\sqrt{n}.
	\]
	Since $\E{x_i \cdot \ind(i \in V)}/\P{i \in V} = \E{x_i \mid i \in V}$ and
	$\P{i \in V} = \int_{s^* - \rho\hat\sigma}^{s^* + \rho\hat\sigma}f(x)dx$, this
	means
	\[
		\left|\E{x_i \mid i \in V} - \mu\right| \leq 2\sigma/\sqrt{n}.
	\]
	By $y_i = x_i - s^*$ for $i \in V$,
	\[
		\left|\E{y_i \mid i \in V} - (\mu - s^*)\right| \leq 2\sigma/\sqrt{n}.
	\]
	We can therefore decompose
	\begin{align*}
		\left|\frac{1}{k_2}\sum_{i \in U_2^{j_1,j_2}} y_i - (\mu - s^*)\right|
		\leq&\; \left|\frac{1}{k_2}\sum_{i \in V}(y_i - (\mu-s^*))\right|
		+ \left|\frac{1}{k_2}\sum_{i \in V^c}(y_i - (\mu-s^*))\right| \\
		\leq&\; \left[\frac{2\sigma}{\sqrt{n}}
		+ \rho\hat\sigma\sqrt{\frac{2\log(4/\beta)}{k_2}}\right]
		+ \frac{2\rho\hat\sigma}{\sqrt{k_2}} \\
		=&\; O\left(\sigma\sqrt{\frac{\log(1/\beta)\log(n)}{k_2}}\right)
	\end{align*}
	where the the first inequality uses a (with probability at least $1-\beta/2$) 
	Chernoff bound on $\{y_i \mid i \in V\}$ concentrating around
	$\E{y_i \mid i \in V}$ as well as $|V^c| \leq 2\sqrt{k_2}$, and the last step
	uses $\hat \sigma \in [\sigma,8\sigma]$.
	
	Next, since we can decompose
	\[
		\frac{1}{k_2}\sum_{i \in U_2^{j_1,j_2}} \tilde y_i
		= \frac{1}{k_2}\sum_{i \in U_2^{j_1,j_2}} y_i
		+ \frac{1}{k_2}\sum_{i \in U_2^{j_1,j_2}} \eta_i
	\]
	where each $\eta_i \sim \Lap{\rho\hat \sigma/\eps}$, the same concentration of
	Laplace noise from Lemma~\ref{lem:hatmutwo} says that with probability
	$1-\beta/2$,
	\[
		\left|\frac{1}{k_2} \sum_{i=1}^{k_2} \eta_i\right|
		= O\left(\frac{\rho\hat\sigma}{\eps}\sqrt{\frac{\log(1/\beta)}{k_2}}\right)
		= O\left(\frac{\sigma}{\eps}\sqrt{\frac{\log(1/\beta)\log(n)}{k_2}}\right).
	\]
	Combining with the bound above and substituting in 
	$k_2 = \Omega\left(\tfrac{n}{\log\left(\frac{\sigma_{\max}}{\sigma_{\min}}+1\right)\sqrt{\log(n)}}\right)$,
	\[
		\left|\frac{1}{k_2} \sum_{i \in U_2^{j_1,j_2}} \tilde y_i
		- (\mu - s^*)\right|
		= O\left(\frac{\sigma}{\eps}\sqrt{\frac{\log\left(\frac{\sigma_{\max}}{\sigma_{\min}} + 1\right)
		\log(1/\beta)\log^{3/2}(n)}{n}}\right).
	\]
	The claim then follows from $\hat \mu_2 = s^* +
	\tfrac{1}{k_2} \sum_{i \in U_2^{j_1,j_2}} \tilde y_i$.
\end{proof}

\section{Proofs from Section~\ref{sec:lower}}
\label{subsec:lower_supp}
For completeness, we start with the formal notion of sequential interactivity used by~\citet{DJW13}, which requires that the set of messages $\{Y_i\}$ sent by the users satisfies the following conditional independence structure: $\{X_i, Y_1, \ldots , Y_{i-1}\} \rightarrow Y_i$ and $Y_i \perp X_j \mid \{X_i , Y_1, \ldots , Y_{i-1}\} \mbox{ for } j\neq i$. Our notion of sequential interactivity --- where each user only sends one message --- is a specific case of this general definition. Our upper bounds all meet this specific requirement, while our lower bound meets the general one.

We start by defining an instance $\est{n,M,\sigma}$. Here, a protocol receives $n$ samples from a $N(\mu,\sigma^2)$ distribution where $\sigma$ is known, $\mu \in [0,M]$, and the goal is to estimate $\mu$. Next, define uniform random variable $V \sim_U \{0,1\}$. Consider the following testing problem: for $V=v$, if $v = 0$, then each user $i$ draws a sample $x_i \sim_{iid} N(0,\sigma^2)$, while if $v = 1$ then each user $i$ draws a sample $x_i \sim_{iid} N(M, \sigma^2)$. The  problem $\test{n,M,\sigma}$ is to recover $v$ from $x_1, \ldots, x_n$. We say protocol $\A$ $(\alpha, \beta)$-solves $\est{n,M,\sigma}$ if, with probability at least $1-\beta$, $\A(\est{n,M,\sigma}) = \hat \mu$ such that $|\hat \mu - \mu| < \alpha$. We will say that an algorithm $\A$ $\beta$-solves $\test{n,M,\sigma}$ if, with probability at least $1-\beta$, $\A(\test{n,M,\sigma}) = v$. Formally, $\test{n,M,\sigma}$ is no harder than $\est{n,M,\sigma}$. 

\begin{lemma}
\label{lem:estest}
	If there exists a sequentially interactive and $(\eps,\delta)$-locally private
	protocol $\A$ that $(M/2, \beta)$-solves $\est{n,M,\sigma}$, then
	there exists a sequentially interactive and $(\eps,\delta)$-locally private
	protocol $\A'$ that $\beta$-solves $\test{n,M,\sigma}$.
\end{lemma}
\begin{proof}
	Let $x_1, \ldots, x_n$ be the samples from an instance of
	$\test{n,M,\sigma}$. We define $\A'$ to run $\A(x_1,\ldots,x_n)$ and then
	output $\arg \min_{\hat \mu \in \{0,M\}} |\A(x_1,\ldots, x_n) - \hat \mu|$.
	Since $\A$ $(M/2, \beta)$-solves $\est{n,M,\sigma}$, with probability
	at least $1-\beta$, $|\A(x_1,\ldots, x_n)- \mu| < M/2$. Thus with probability
	at least $1-\beta$, $\A'(x_1, \ldots, x_n) = v$. Thus $\A'$ $\beta$-solves
	$\test{n,M,\sigma}$. As $\A'$ interacted with $x_1,\ldots,x_n$ only through
	$(\eps,\delta)$-locally private $\A$, by preservation of differential privacy
	under postprocessing, $\A'$ is $(\eps,\delta)$-locally private as well.
	Similar logic implies that $\A'$ is also sequentially interactive.
\end{proof}

We now extend this result to $(\eps,\delta)$-locally private protocols using results from both~\citet{BNS18} and~\citet{CSUZZ18}\footnote{Both of these results are stated for noninteractive protocols, it is straightforward to see that their techniques carry over to sequentially interactive protocols. This is because both results rely on transforming a single user call to an $(\eps,\delta)$-local randomizer into calls to an $(O(\eps),0)$-local randomizer. Since users in sequentially interactive protocols still only make a single call to a local randomizer, we can apply the same transformations to each single user call and obtain an $(O(\eps),0)$-locally private sequentially interactive protocol.}.

\begin{lemma}
\label{lem:test_hard_appx}
	Let $\delta < \min\left(\frac{\epsilon\beta}{48n\ln(2n/\beta)},	\frac{\beta}{16n\ln(n/\beta)e^{7\eps}}\right)$, $\eps > 0$, and suppose that $\A$ is a sequentially interactive and $(\eps,\delta)$-locally private protocol. If $\A$ $\beta$-solves $\test{n,M,\sigma}$, then there exists a sequentially interactive $(10\eps,0)$-locally private $\A'$ that $4\beta$-solves	$\test{n,M,\sigma}$. 
\end{lemma}
\begin{proof}
	Our analysis splits into two cases depending on $\epsilon$.
	
	\underline{Case 1}: $\eps \leq 1/4$. In this case, we use a result
	from~\citet{BNS18}, included here for completeness.
	\begin{fact}[Theorem 6.1 in~\citet{BNS18} (restated)]
		Given $\eps \leq 1/4$ and $\delta < \epsilon\beta/48n\ln(2n/\beta)$,
		there exists a $(10\eps,0)$-locally private algorithm $\A'$ such that for
		every database $U = \{x_1,\ldots, x_n\}$,
		$d_{TV}(\A(U), \A'(U)) \leq \beta$, where $d_{TV}$ denotes total variation
		distance.
	\end{fact}
	Thus, denoting by $E_{\A}$ the event where $\A$ recovers the correct $v$ on
	$\test{n,M,\sigma}$ and $E_{\A'}$ the event where $\A'$ recovers the correct $v$ on
	$\test{n,M,\sigma}$, $|\P{E_{\A}} - \P{E_{\A'}}| \leq \beta$, where the
	probabilities are respectively over $\A$ and $\A'$. Thus since $\A$
	$\beta$-solves $\test{n,M,\sigma}$, it follows that $\A'$ $2\beta$-solves (and thus
	also $4\beta$-solves) $\test{n,M,\sigma}$. 
	
	\underline{Case 2}: $\eps > 1/4$. In this case we use a result
	from~\citet{CSUZZ18}\footnote{~\citet{CSUZZ18} originally state their result
	for $\eps > 2/3$, but mildly strengthening their assumed upper bound on
	$\delta$ from $\delta < \frac{\beta}{8n\ln(n/\beta)e^{6\eps}}$ to 
	$\delta < \frac{\beta}{16n\ln(n/\beta)e^{7\eps}}$ yields the result here.}
	\begin{fact}[Theorem A.1 in~\citet{CSUZZ18} (restated)]
		Given $\eps > 1/4$ and $\delta < \frac{\beta}{16n\ln(n/\beta)e^{7\eps}}$,
		there exists an $(8\eps,0)$-locally private protocol $\A'$ such that 
		$\A'$ $4\beta$-solves $\test{n,M,\sigma}$.
	\end{fact}
\end{proof}

Finally, we prove that $\testno$ is hard for $(\eps,0)$-locally private protocols. At a high level, we prove this result by viewing $\testno$ as a Markov chain $V \to$ data $X \to$ outputs $Y \to$ answer $Z$. We bound the mutual information $I(V;Z)$ by a function of $M,\sigma,$ and $I(X;Y)$ using a strong data processing inequality for Gaussian distributions (see Section 4.1 in~\citet{BGMNW16} or~\citet{R16} for details; a primer on information theory appears in the last section). We further bound $I(X;Y)$ using existing tools from the privacy literature~\cite{DJW13}. The resulting upper bound on $I(V;Z)$ enables us to lower bound the probability of an incorrect answer $Z$.

\begin{lemma}
\label{lem:test_hard}
	Suppose $M \leq \sigma/[4(e^\eps-1)\sqrt{2nc}]$,where $c$ is an absolute
	constant. For any sequentially interactive and $(\eps,0)$-locally private
	protocol $\A$ that $\beta$-solves $\test{n,M,\sigma}$, $\beta \geq 1/4$.
\end{lemma}
\begin{proof}
  	We may express any sequentially interactive $(\eps,0)$-locally private
  	protocol $\A$ that $\beta$-solves $\test{n,M,\sigma}$ as a Markov chain
  	$V \to X \to Y \to Z$, where $V$ is the random variable selecting
  	$v$, $X = (x_1, \ldots, x_n)$ is the random variable for users' i.i.d.
  	samples, $Y = (y_1, \ldots, y_n)$ is the random variable for users'
  	$(\eps,0)$-privatized responses, and $Z = \A(\test{n,M,\sigma})$.  As
  	$V \to X \to Y \to Z$ is a Markov chain (i.e., any two random variables
  	in the chain are conditionally independent given a random variable between
  	them). Thus by a strong data processing inequality for two Gaussians (see
  	e.g. Section 4.1 in~\citet{BGMNW16} or, for a broader treatment of strong
  	data processing inequalities, \citet{R16}), there exists absolute constant
  	$c$ such that for each user $i$,
  	$I(V;Y_i) \leq \tfrac{cM^2}{\sigma^2}I(X_i;Y_i)$, where $I(A;B)$ denotes 
  	the mutual information between random variables $A$ and $B$. 
  	Next, since our protocol is $(\eps,0)$-locally private, by Corollary 1
  	from~\citet{DJW13}, for each user $i$, $I(X_i;Y_i) \leq 4(e^\eps-1)^2$. With
  	the equation above, we get
  	
  	\begin{equation}
  	\label{eq:info}
  		I(V;Y_i) \leq \tfrac{4cM^2(e^\eps-1)^2}{\sigma^2}.
  	\end{equation}
  	
  	Without loss of generality, suppose $Z$ is a deterministic function of $Y$ (if $Z$ is a random function of $Y$ then it 	decomposes into a convex combination of deterministic functions of $Y$). From Markov chain
  	$V \to X \to Y \to Z$ and the (generic) data processing inequality we get
  	\begin{align*}
  		I(V;Z) \leq&\; I(V;Y_1, \ldots, Y_n) \\
  		=&\; \sum_{i=1}^n I(V;Y_i \mid Y_{i-1}, \ldots Y_1) \\ 
  		\leq&\; \sum_{i=1}^n I(V,Y_{i-1}, \ldots, Y_1;Y_i) \\
  		=&\; \sum_{i=1}^n \left[I(V;Y_i) + I(Y_{i-1}, \ldots, Y_1 ; Y_i|V)\right] \\
  		=&\; \sum_{i=1}^n I(V;Y_i)
  	\end{align*}
  	where the last step follows from the independence of $Y_i$ and
  	$Y_1, \ldots, Y_{i-1}$ given $V$. Substituting in Equation~\ref{eq:info},
  	$I(V;Z) \leq \tfrac{4ncM^2(e^\eps-1)^2}{\sigma^2}$. Therefore by
  	$M \leq \sigma/4(e^\eps-1)\sqrt{2nc}$ we get $I(V;Z) \leq 1/8$.
  	
  	Define $P$ to be the distribution of $Z$ (over the randomness of $V$, $X$, 
  	and $Y$), and let $P_0$ and $P_1$ be the distributions for $Z|V=0$ and
  	$Z|V=1$ respectively. Then as $V$ is uniform, $P = (P_0 + P_1)/2$, so
  	\[
	  	||P - P_0||_1 = ||P - P_1||_1 = \tfrac{1}{2}||P_0 - P_1||_1.
  	\]
  	Moreover, by
  	\begin{align*}
	  	\P{Z=V} =&\; \P{Z=0,V=0} + \P{Z=1,V=1} \\
  		=&\; \frac{1}{2}(P_0(0) + [1 - P_1(0)]) \\
  		\leq&\; \frac{1}{2}(1 + |P_0(0) - P_1(0)|) \\
  		=&\; \frac{1}{2} + \frac{1}{4}||P_0 - P_1||_1
  	\end{align*}
  	we get $\P{Z=V} \leq \tfrac{1}{2} + \tfrac{1}{4}||P_0 - P_1||_1$. Thus
  	\begin{align*}
  		\frac{||P_0 - P_1||_1^2}{8} =&\; \frac{1}{4}(||P_0 - P||_1^2
  		+ ||P_1 - P||_1^2) \\
  		\leq&\; \frac{1}{2}(D_{KL}(P_0 || P) + D_{KL}(P_1 || P)) \\
  		=&\; I(Z;V) \leq 1/8
  	\end{align*}
  	where the second-to-last inequality uses Pinsker's inequality. It follows
  	that $||P_0 - P_1||_1 \leq 1$. Substituting this into
  	$\P{Z=V} \leq \tfrac{1}{2} + \tfrac{1}{4}||P_0 - P_1||_1$, we get
  	$\P{Z=V} \leq \tfrac{3}{4}$.
\end{proof}

We combine the preceding results to prove a general lower bound for $\estno$ as follows: for appropriate $\eps$ and $\delta$, by Lemma~\ref{lem:estest} any sequentially interactive and $(\tfrac{\eps}{10},\delta)$-locally private protocol $\A$ that $(M/2,\tfrac{\beta}{4})$-solves $\est{n,M,\sigma}$ implies the existence of a sequentially interactive and $(\tfrac{\eps}{10},\delta)$-locally private protocol $\A'$ that $\tfrac{\beta}{4}$-solves $\test{n,M,\sigma}$. Then, Lemma~\ref{lem:test_hard_appx} implies the existence of a sequentially interactive and $(\eps,0)$-locally private protocol $\A''$ that $\beta$-solves $\test{n,M,\sigma}$. By Lemma~\ref{lem:test_hard} any such $\A'$ that $\beta$-solves $\test{n,M,\sigma}$ has $\beta \geq 1/4$. Hardness for $\testno$ therefore implies hardness for $\estno$. We condense this reasoning into the following theorem.

\begin{theorem}
	Let $\delta < \min\left(\frac{\epsilon\beta}{60n\ln(5n/2\beta)}, \frac{\beta}{16n\ln(n/\beta)e^{7\eps}}\right)$, $\eps > 0$, and let $\A$ be a sequentially interactive $(\eps,\delta)$-locally private $(\alpha,\beta)$-estimator for $\est{n,M,\sigma}$ where $M = \sigma/[4(e^\eps-1)\sqrt{2nc}]$, $c$ is as in Lemma~\ref{lem:test_hard}, and $\beta <1/16$. Then $\alpha \geq M/2 = \Omega\left(\frac{\sigma}{\eps}\sqrt{\frac{1}{n}}\right)$.
\end{theorem}

In particular, Theorem~\ref{thm:lb_formal} implies that our upper bounds are tight up to logarithmic factors for \emph{any} sequentially interactive and $(\eps,\delta)$-locally private protocol with sufficiently small $\delta$. Using recent subsequent work~\cite{JMNR19}, we can also extend this result to the fully interactive setting, as shown in the next section.

\subsection{Extension to Fully Interactive Lower Bound}
The following result, proven in subsequent work by~\citet{JMNR19} also relying on the work of~\citet{BGMNW16}, gives a general lower bound for locally private simple hypothesis testing problems like $\testno$. 

\begin{lemma}[Theorem 5.3 in~\citet{JMNR19}]
\label{lem:sifi}
	For $\eps > 0$ and $\delta < \min\left(\tfrac{\eps^3\alpha^2}{48n\ln(2n/\beta)}, \frac{\eps^2\alpha^2}{64n\ln(n/\beta)e^{7\eps}}\right)$, any $(\eps,\delta)$-locally private simple hypothesis testing protocol distinguishing between distributions $P_0$ and $P_1$ with probability at least $2/3$ requires $n = \Omega\left(\tfrac{1}{\eps^2 \tv{P_0}{P_1}^2}\right)$ samples.
\end{lemma}

Since in general $D_{KL}(N(\mu_1, \sigma^2) || N(\mu_2, \sigma^2)) \leq \left[\tfrac{\mu_1 - \mu_2}{\sigma}\right]^2$, in the setting of $\test{n,M,\sigma}$ we are distinguishing between $P_0 = N(0,\sigma^2)$ and $P_1 = N(M,\sigma^2)$ and get $D_{KL}(P_0 || P_1) = O\left(\tfrac{M^2}{\sigma^2}\right)$. Pinsker's inequality then implies $\tv{P_0}{P_1}^2 = O\left(\tfrac{M^2}{\sigma^2}\right)$. Substituting this into Lemma~\ref{lem:sifi}, we get that distinguishing $P_0$ and $P_1$ with constant probability and $n$ samples requires $M = \Omega\left(\tfrac{\sigma}{\eps \sqrt{n}}\right)$. Thus, for appropriately small $\delta$, any $(\eps,\delta)$-locally private protocol that $(\alpha,\beta)$-solves $\est{n,M,\sigma}$ has $\alpha = \Omega(M) = \Omega\left(\tfrac{\sigma}{\eps \sqrt{n}}\right)$.

\section{Information Theory Overview}
We briefly review some standard facts and definitions from information theory, starting with entropy.

\begin{definition}
	The \emph{entropy} $H(X)$ of a random variable $X$ is
	\[
		H(X) = \sum_x \P{X=x} \ln\left(\tfrac{1}{\P{X=x}}\right),
	\]
	and the \emph{conditional entropy} $H(X|Y)$ of random variable $X$ conditioned
	on random variable $Y$ is 
	\[
		H(X|Y) = \mathbb{E}_y[H(X|Y = y)].
	\] 
\end{definition}

Next, we can use entropy to define the mutual information between two random variables. Mutual information between random variables $X$ and $Y$ is roughly the amount by which conditioning on $Y$ reduces the entropy of $X$ (and vice-versa).

\begin{definition}
\label{def:muinfo}
	The \emph{mutual information} $I(X;Y)$ between two random variables $X$ and $Y$ is
	\[
		I(X;Y) = H(X) - H(X|Y) = H(Y) - H(Y|X),
	\]
	and the \emph{conditional mutual information} $I(X;Y|Z)$ between $X$ and $Y$
	given $Z$ is
	\[
		I(X;Y|Z) = H(X|Z) - H(X|Y,Z) = H(Y|Z) - H(Y|X,Z).
	\] 
\end{definition}

We also define the related notion of KL-divergence.

\begin{definition}
	The \emph{Kullback-Leibler divergence} $D_{KL}(X||Y)$ between two random
	variables $X$ and $Y$ is
	\[	
		D_{KL}(X || Y) = \sum_x \P{X = x} \ln\left(\frac{\P{X = x}}{\P{Y = x}}\right),
	\]
	where we often abuse notation and let $X$ and $Y$ denote the distributions
	associated with $X$ and $Y$.
\end{definition}

KL divergence connects to mutual information as follows.

\begin{fact}
\label{fact:div}
	For random variables $X$, $Y$, and $Z$,
	\[
		I(X;Y|Z)
		= \mathbb{E}_{x,z}\left[D_{KL}\left((Y|X = x, Z=z)\|(Y|Z=z)\right)\right].
	\]
\end{fact}

Finally, we will also use the following connection between KL divergence and $||\cdot||_1$ distance.

\begin{lemma}[Pinsker's inequality]
	For random variables $X$ and $Y$,
	\[	
		||X-Y||_1 \leq \sqrt{2D_{KL}(X||Y)}.
	\]
\end{lemma}

\end{document}